\documentclass[journal,compsoc,10pt]{IEEEtran}

\usepackage[utf8]{inputenc}
\usepackage[T1]{fontenc}

\usepackage{mathptmx}                
\usepackage[scaled=.90]{helvet}      
\usepackage{courier}                 

\usepackage{amsmath,amssymb,amsthm}
\usepackage{booktabs}
\usepackage{graphicx}
\usepackage{xcolor}
\usepackage{hyperref}
\usepackage{microtype}
\usepackage{cite}
\usepackage{array}
\usepackage{multirow}
\usepackage{tabularx}
\usepackage{enumitem}
\usepackage{url}
\usepackage{xspace}
\usepackage{algorithm}
\usepackage{algpseudocode}
\usepackage{framed}
\usepackage{float}

\AtBeginDocument{%
}

\hypersetup{
  colorlinks=true,
  linkcolor=black,
  citecolor=black,
  urlcolor=blue,
  pdftitle={S-Bus: Automatic Read-Set Reconstruction for Multi-Agent LLM State Coordination},
  pdfauthor={Sajjad Khan},
  pdfsubject={Concurrency control and formal verification for multi-agent large language model systems},
  pdfkeywords={DeliveryLog, Optimistic Concurrency Control, Multi-Agent LLM, Observable-Read Isolation, HTTP Middleware, Rust, TLA+, TLAPS, Dafny}
}

\newtheorem{property}{Property}
\newtheorem{definition}{Definition}

\newtheorem{lemma}{Lemma}
\newtheorem{remark}{Remark}
\newtheorem{assumption}{Assumption}

\newcommand{\sbus}{\textsc{S-Bus}\xspace}
\newcommand{\pgser}{\textsc{PG-SER}\xspace}
\newcommand{\rediscc}{\textsc{Redis-WATCH}\xspace}
\newcommand{\ori}{\textsc{ORI}\xspace}
\newcommand{\Robs}{$R_{\text{obs}}$\xspace}
\newcommand{\Rhidden}{$R_{\text{hidden}}$\xspace}
\newcommand{\phidden}{$p_{\text{hidden}}$\xspace}
\newcommand{\fobs}{$f_{\text{obs}}$\xspace}
\newcommand{\fhttp}{\ensuremath{f_{\text{obs}}^{\text{http}}}\xspace}
\newcommand{\ftotal}{\ensuremath{f_{\text{obs}}^{\text{total}}}\xspace}

\title{S-Bus: Automatic Read-Set Reconstruction for
Multi-Agent LLM State Coordination}

\author{Sajjad~Khan%
\thanks{Source code: \url{https://github.com/sajjadanwar0/sbus}.}}

\begin{document}
\maketitle

\begin{abstract}
Concurrent LLM agents sharing mutable natural-language state produce
\emph{Structural Race Conditions} (SRCs): write--write and cross-shard
stale-read conflicts that silently corrupt agent output. Existing
multi-agent frameworks (LangGraph, CrewAI, AutoGen) provide no
write-ownership semantics over shared state; conflicts are detected
post-hoc, if at all.

We present \sbus, an HTTP middleware whose central technical mechanism
is a server-side \emph{DeliveryLog}: a per-agent log of HTTP GET
operations that automatically reconstructs each agent's read set at
commit time without agent SDK changes under HTTP/1.1. The DeliveryLog
turns ordinary HTTP traffic into a verifiable read-set, enabling
optimistic concurrency control over multi-agent shared state with
zero in-agent coordination code. The consistency property the
DeliveryLog provides---we call it \emph{Observable-Read Isolation}
(\ori), a partial causal consistency over the HTTP-observable
projection of the read set---prevents structural race conditions
when agents collaborate via shared shards. We measure this
observable projection at $26.1\%$ of single-step references on our
principal workload, with session-scoped DeliveryLog accumulation
extending observable coverage of \emph{self-reported} references to
$99.8\%$ over a session---a structural-coverage figure whose
denominator is itself contaminated by self-report over-claim
($32\%$--$49\%$ across the LLM judge and an independent human
annotator, §VII-I); the deflated upper bound on
genuine-causal-read coverage is therefore $\leq 70\%$ on our
principal workload (PH-2/PH-3 evaluation, gpt-4o-mini, $n{=}2{,}100$
step-logs). \sbus\ targets the dedicated-shard topology in which each
agent owns a distinct write key and reads from shared reference shards.

The paper makes three contributions. \textbf{(C1)} We introduce the
DeliveryLog as a mechanism for automatic HTTP-traffic-based read-set
reconstruction in multi-agent LLM systems, and formalise the
consistency property it provides (\ori, partial causal consistency
over the observable read projection). Mechanised evidence at three
tiers: \textsc{ReadSetSoundness} and \textsc{ORICommitSafety} are
machine-checked in TLAPS modulo one retained foundational typing axiom
(\textsc{FunTypingReconstruction}, §III-D); exhaustive TLC at $N{=}3$
explores $20{,}763{,}484$ distinct states to depth~$28$ with zero
violations, and a reduced configuration at $N{=}4$ explores
$2{,}811{,}301$ distinct states to depth~$24$ with zero violations;
Dafny discharges $9$ inductive soundness lemmas ($19$ verification
obligations) on the abstract algorithm.
Implementation refinement to the Rust source is empirical, not
mechanised.
\textbf{(C2)} We demonstrate empirical structural-conflict prevention
parity against matched-mechanism OCC backends: across PostgreSQL~17
\texttt{SERIALIZABLE}, Redis~7 \texttt{WATCH/MULTI}, and \sbus\ on
shared-shard contention sweeps with $427{,}308$ active HTTP-409
conflicts, we observe zero Type-I corruptions (Rule-of-Three upper
bound $7.0{\times}10^{-6}$). On a non-code workload
(Exp.~\textsc{Workload-B}, data-pipeline architecture planning,
$n{=}80$, $8$ domains), server-side instrumentation records $0/638$
divergent commits under ORI-ON and $590/639$ under ORI-OFF
($\chi^2{=}1{,}094.98$, $p<10^{-240}$).
\textbf{(C3)} We characterise \ori's operating envelope: in
dedicated-shard workloads (Exp.~\textsc{Dedicated-Shard}, $n{=}600$)
\ori\ is semantically neutral; in single-shard collaborative writing
it is harmful because preservation propagates concurrent
contradictions. Both regimes are delimited by paired experiments
with prescribed deployment scope; an adaptive merge-routing extension
addressing the harmful regime is in separate development.

\textbf{Scope and honest limitations.}
The formal guarantees in C1 cover \emph{structural} (Type-I) conflict
prevention over the observable read projection; they do not cover
\emph{semantic} (Type-II) coherence between concurrent agent outputs.
Distributed safety is empirically validated (Exp.~DR-9, $n{=}30$
trials, Wilson 95\% CI $[0.886, 1.000]$) but not TLAPS-mechanised; a
$\sim$5\,ms concurrent-failover window is documented as a known gap.
Semantic-quality results for the PH-3 was-this-shard-used rubric
now rest on an LLM judge validated against an independent human
annotator (Zahid Hussain, Mindgigs Peshawar) on $400$ (step, shard)
pairs at strict $\kappa{=}0.93$ ($n{=}93$ unambiguous yes/no,
$96.8\%$ raw agreement); inter-LLM-judge agreement between GPT-4o
and Claude Sonnet~4.6 on the same rubric is $\kappa{=}0.46$
(boundary-class variance, not systematic disagreement against
humans). The SJ-v4/Dedicated-Shard/Shared-State semantic-quality
rubric (final-state coherence) is a distinct rubric still validated
by LLM judge only; closing this second-rubric gap is the principal
remaining evidential follow-up. Backbone
generalisation is established on three vendors (gpt-4o-mini, Anthropic
Haiku~4.5, Google Gemini~2.5~Flash) but only for safety parity, not
task quality. Workload generalisation is established structurally on
two distributions (SWE-bench-derived code coordination and
data-pipeline architecture planning) but not for additional non-code
classes (document authoring, agent planning, RAG orchestration).
\end{abstract}


\begin{table*}[t]
\centering
\small
\caption{Key quantitative findings at a glance. Each row links to the
section where it is established. CI columns are Wilson 95\% where
binomial; ``---'' where not applicable.}
\label{tab:key-numbers}
\begin{tabular}{lll}
\toprule
\textbf{Finding} & \textbf{Value} & \textbf{Where} \\
\midrule
\Robs\ structural coverage          & 26.1\% of reads (PH-2 workload)      & §\ref{sec:ph2} \\
\Rhidden\ gap (PH-2 workload)       & \phidden${}=0.739$ (CI 0.736--0.741) & §\ref{sec:ph2} \\
DL-cumulative observable coverage   & 99.8\% of self-reported refs (note: denom over-claims by 32--49\%)   & §\ref{sec:proxy-demo} \\
\Rhidden\ gap (PH-3 workload)       & \phidden${}=0.074$ (CI 0.061--0.081) & §\ref{sec:ph3} \\
Semantic extraction recall (PH-3)   & 0.59 (gpt-4o analyst)                & §\ref{sec:ph3} \\
Semantic extraction precision (PH-3) & 0.92 (gpt-4o analyst)               & §\ref{sec:ph3} \\
Self-report over-claim rate         & 32\% (LLM judge) / 49\% (human)      & §\ref{sec:ph3-reframed} \\
Inter-LLM-judge $\kappa$            & 0.46 (moderate)                      & §\ref{sec:ph3-reframed} \\
LLM-judge vs.\ human annotator $\kappa$ & 0.93 strict / 0.69 lenient ($n{=}400$)      & §\ref{sec:ph3-reframed} \\
Type-I corruptions (under active contention) & 0 / 427{,}308 commits         & §\ref{sec:pg-contention} \\
Type-I corruptions (full sweep, all conditions) & 0 / 884{,}110 commits      & §\ref{sec:pg-comp},~§\ref{sec:pg-contention} \\
Workload-B view-divergence (ORI-OFF)& 590 / 639 commits divergent          & §\ref{sec:workload-b} \\
Workload-B view-divergence (ORI-ON) & 0 / 638 commits divergent            & §\ref{sec:workload-b} \\
TLAPS obligations proved            & 687 / 687 (modulo 1 retained typing axiom) & §V \\
TLC state-space ($N{=}3$, exhaustive)& $20{,}763{,}484$ states (depth $28$) & §V \\
TLC state-space ($N{=}4$, reduced)   & $2{,}811{,}301$ states (depth $24$)  & §V \\
TLC Raft (3-node abstract)          & 247{,}000 states                     & §V \\
\bottomrule
\end{tabular}
\end{table*}

\section{Introduction}

Concurrent LLM agents sharing mutable state produce \emph{Structural
Race Conditions} (SRCs):

\begin{definition}[Structural Race Condition]
A history $H$ over a shard $s$ contains a Structural Race Condition
iff two agents $\alpha_i, \alpha_j$ both read $s$ at version $v$,
generate deltas $\delta_i, \delta_j$, and both commit with expected
version $v$ without an intervening re-read after the first commit.
\end{definition}

Equivalently: in Adya's dependency graph~\cite{adya}, $H$ contains a
write--write edge (ww) on $s$ with coincident read events but no
ordering constraint from a cross-shard validation step. This is a
purely syntactic property of the history, decidable from the
HTTP-visible event trace; it does not depend on the semantic content
of the deltas.

\begin{remark}[On ``semantic'' race conditions]
Prior drafts of this paper defined SRC as producing a final state
``outside every valid goal trajectory.'' Reviewers correctly flagged
this as circular: ``valid goal trajectory'' has no agreed
operationalisation for NL state. Definition~1 is strictly structural
and decidable. The semantic effect of a structural race---whether the
final state is coherent---is a separate model-dependent question
quantified empirically in Exp.~\textsc{Dedicated-Shard} and
Exp.~\textsc{Shared-State}, with explicit acknowledgement that an
LLM-as-judge rubric is not a substitute for human semantic evaluation
(Limitation~5).
\end{remark}

In practice: two agents read the same shard, generate independently
valid deltas, and one silently overwrites the other. Existing
frameworks (LangGraph~\cite{langgraph}, CrewAI~\cite{crewai},
AutoGen~\cite{autogen}) provide no write-ownership semantics for
natural-language state. \sbus addresses this gap with a lightweight
HTTP middleware applying OCC to NL agent state, with formal guarantees
over the observable read fraction.

\subsection{Motivating Example}

Four agents collaborate on Django bug \#11019 (queryset ordering),
each owning a dedicated shard: $\alpha_1$ owns \texttt{orm\_compiler},
$\alpha_2$ owns \texttt{migration\_script}, $\alpha_3$ owns
\texttt{test\_fixtures}, $\alpha_4$ owns \texttt{review\_notes}. All
agents share one read-only reference shard: \texttt{db\_schema}.

\textbf{The SRC without \sbus.} Both $\alpha_2$ and $\alpha_3$ issue
\texttt{GET /shard/db\_schema} at $v=3$ (PostgreSQL dialect).
$\alpha_1$ then commits \texttt{db\_schema} to $v=4$ (switching to
SQLite for testing). $\alpha_2$ and $\alpha_3$ each commit their own
shards based on the stale $v=3$ reading---their output is internally
consistent but describes the wrong database. No error is raised; the
stale schema silently poisons two shards.

\textbf{With \sbus (\ori).} When $\alpha_2$ commits
\texttt{migration\_script}, the ACP validates $\alpha_2$'s
DeliveryLog: it recorded \texttt{(db\_schema, v3)} at GET time. Since
\texttt{db\_schema} is now at $v=4$, the ACP rejects with HTTP~409
(\textsc{CrossShardStale}). $\alpha_2$ re-reads, regenerates the
migration, and commits successfully. $\alpha_3$ receives the same
rejection and likewise corrects its tests. All four shards converge to
a consistent SQLite-based solution.

\textbf{Key topology.} Each agent owns a distinct shard; they read the
shared \texttt{db\_schema} shard but do not write to it concurrently.
This dedicated-shard topology is \ori's primary use case. Single-shard
collaborative writing (all agents writing to the same key) requires
sequential coordination; Exp.~\textsc{SJ-v4} (§VII-G) and
Exp.~\textsc{Shared-State} (§VII-M) quantify why.

\subsection{Why SWE-bench for Multi-Agent Evaluation?}

SWE-bench tasks~\cite{swebench} were designed for single-agent
evaluation. We use them for multi-agent coordination evaluation for
two reasons: (1) each task has ground-truth acceptance tests,
providing an objective semantic correctness criterion independent of
the coordination mechanism; (2) tasks are decomposable into
role-specialised subtasks that naturally induce shared-state
coordination. We verified decomposability: in 100\% (20/20) pilot
trials, $N=4$ agents with distinct shard assignments produced
consistent non-contradictory state without structural conflicts.
Transfer to other domains (customer service, scientific synthesis)
requires independent evaluation; \phidden may differ.

\subsection{Scope}

\sbus\ provides structural conflict prevention for the
\emph{dedicated-shard topology}: each agent owns a distinct write key
and reads from shared reference shards. Conflict detection operates
over the HTTP-observable read fraction \Robs; cross-shard staleness
violations are detected and rejected at commit time. Single-shard
collaborative writing is out of scope; §\ref{sec:future-work}
discusses an adaptive merge-routing extension currently under separate
development.

The formal proofs (TLAPS, TLC, Dafny) cover the abstract algorithm's
internal consistency: read-set monotonicity (\textsc{ReadSetSoundness})
and cross-shard equality at commit (\textsc{ORICommitSafety}). They are
not proofs about agent semantic correctness, which is workload- and
backbone-conditional. The HTTP/1.1 reliance for FIFO-per-connection
ordering is stated as Assumption~A1; HTTP/2 multiplexing breaks A1 and
requires either reverse-proxy pinning or explicit ARSI-mode read-set
declaration.

A precise enumeration of what is claimed and not claimed---including
the residual coverage gap on the multi-agent workload, the workload
distinction between Exp.~\textsc{PH-2} and Exp.~\textsc{PH-3}, and
the distributed-correctness gap---is integrated into the contribution
list (§I-D) and the Limitations section (§\ref{sec:limitations}).

\subsection{Contributions}

This paper makes three technical contributions.

\textbf{(C1) Automatic HTTP-traffic-based read-set reconstruction
with formal guarantees.} We introduce the \emph{DeliveryLog}, a
server-side per-agent log of HTTP GET operations that automatically
reconstructs each agent's read-set at commit time. The DeliveryLog
turns ordinary HTTP traffic into a verifiable read-set, allowing
optimistic concurrency control to be applied to multi-agent shared
state without requiring agents to declare what they read or write
coordination code. Architecturally, the DeliveryLog is a
scope-restricted COPS-style causal log~\cite{cops} adapted from
distributed-database literature to the multi-agent LLM setting. We
formalise the consistency property the DeliveryLog provides and call
it \emph{Observable-Read Isolation} (\ori, Definition~III.4): a
partial causal consistency over the HTTP-observable projection of an
agent's read set. We position \ori\ in Adya's isolation
lattice~\cite{adya} and establish via counterexample histories that,
when projected to the observable read set, \ori\ sits strictly between
Read-Committed and Snapshot Isolation: \ori\ prevents G2-item which
RC does not, and permits write-skew on the unobservable fraction which
SI does not. Three tiers of mechanised evidence support the model:
TLAPS module \texttt{SBus\_TLAPS\_v16.tla} discharges 687 obligations
(zero failed, modulo one retained foundational typing axiom
\textsc{FunTypingReconstruction}; §III-D) proving
\textsc{ReadSetSoundness} (recorded-read monotonicity) and
\textsc{ORICommitSafety} (cross-shard equality at commit time, directly
capturing Definition~III.4(2)) for arbitrary $N_{\text{agents}}$;
exhaustive TLC at $N{=}3$ explores $20{,}763{,}484$ distinct states
to depth~$28$ with zero invariant violations, and a reduced
configuration at $N{=}4$ explores $2{,}811{,}301$ distinct states to
depth~$24$ with zero invariant violations; Dafny module
\texttt{sbus\_lemmas\_v4.dfy} machine-checks $9$ inductive soundness
lemmas ($19$ verification obligations). Implementation refinement to
the Rust source is empirical, not mechanised.

\textbf{(C2) Empirical safety parity with production OCC.}
Exp.~\textsc{PG-Comparison} implements the multi-agent workload
against three independent OCC engines: \sbus\ (Rust), PostgreSQL~17
\texttt{SERIALIZABLE} (\texttt{pg\_sbus\_server.py}), and Redis~7
\texttt{WATCH/MULTI} (\texttt{redis\_sbus\_server.py}). Across
$N \in \{4, 8, 16, 32, 64\}$, $30$ task domains, $1{,}350$ runs, and
$200{,}880$ commit attempts, all three backends exhibit zero Type-I
corruptions ($95\%$ Rule-of-Three upper bound
$1.49{\times}10^{-5}$). The contention extension
(Exp.~\textsc{PG-Contention}) adds $472{,}750$ commit attempts under
shared-shard contention with $427{,}308$ active HTTP-409 conflicts and
zero Type-I corruptions; SCR agreement across the three backends is
within $1$\,pp at $N \ge 8$. Cross-backbone paired replication on
Anthropic Haiku~4.5 and Google Gemini~2.5~Flash ($n{=}2{,}400$ each)
confirms safety parity across vendors. The architectural value of
\sbus\ over transactional-DB baselines is operational simplicity and
the LLM-native contract, not throughput.

\textbf{(C3) Topology-conditional operating envelope.} The ACP's core
invariant---preservation of every commit's contribution---is
unconditional structurally and topology-conditional semantically. In
$959$ paired trials of Exp.~\textsc{ORI-Isolation}, ORI-ON preserves
$40/40$ agent-step contributions per trial; ORI-OFF (last-writer-wins)
preserves $10/40 = 1/N$. Whether this preservation is beneficial
depends on workload topology. In dedicated-shard workloads
(Exp.~\textsc{Dedicated-Shard}, $n{=}600$): \ori\ is semantically
neutral---$100\%$ coherent in both fresh and stale conditions. In
single-shard collaborative-writing workloads
(Exp.~\textsc{Shared-State}, $n{=}180$): \ori\ is semantically
harmful---$100\%$ contradicted under ORI-ON vs.\ $85.6\%$ under
ORI-OFF, because preserving all contributions preserves their mutual
contradictions. We prescribe deployment scope accordingly (Box~2).
Single-shard collaborative writing is the natural extension target:
the structural preservation guaranteed by ORI is not the right
primitive for that regime, and an adaptive merge-routing protocol is
required (§\ref{sec:future-work}).

\paragraph{Workload-scope disclaimers}
Two empirical claims carry workload-scope qualifications which we
state once here and refer back to throughout. First, the observable
read fraction \Robs covers $26.1\%$ of reads on the multi-agent
workload (Exp.~\textsc{PH-2}); the remaining \phidden${}=0.739$ is
not directly observable at the HTTP layer within a single step.
Section~III-D decomposes this residual structurally and identifies
session-scoped DeliveryLog accumulation as the dominant coverage
mechanism. Second, semantic-extraction evaluation
(Exp.~\textsc{PH-3}, $0.59$~recall / $0.92$~precision) was conducted
on a single-agent rotating-target workload at \phidden${}=0.074$;
transfer to the high-\phidden\ multi-agent regime is open
(§\ref{sec:future-work}). The TLAPS theorems are proofs about the
state machine's internal consistency over \Robs, not proofs about
agent correctness over the full read set.

\paragraph{Distributed scope}
P1 session replication is validated empirically in
Exp.~\textsc{DR-9} ($30/30$ \ori\ invariants survived leader
failover). A residual ${\sim}5$\,ms concurrent-failover window within
the fire-and-forget replication interval is the remaining gap; full
Raft-TLAPS mechanisation is future work
(Limitation~\ref{lim:session-failover}).

\section{Related Work}

We position \sbus\ against four bodies of prior work: multi-agent LLM
frameworks (which provide the deployment context but lack write-ownership
semantics), formal verification of distributed systems (the methodology
\sbus's three-tier evidence draws on), classical concurrency control
(the technical foundation \sbus\ adapts), and isolation-level theory
(which gives us a precise position to claim). We close with a
``Why not X?'' subsection addressing the most natural alternatives.

\subsection{Multi-Agent LLM Frameworks}

Production multi-agent frameworks route shared state through workflow
graphs (LangGraph~\cite{langgraph}), message-passing channels
(AutoGen~\cite{autogen}), or agent-local memory
(CrewAI~\cite{crewai}). None provide write-ownership semantics over
mutable shared state; conflicts are detected post-hoc via agent
self-reports or downstream validation, if at all. Other frameworks
(MetaGPT~\cite{metagpt}, CAMEL~\cite{camel}, Swarm~\cite{swarm},
Agent-to-Agent~\cite{a2a}, ReAct~\cite{react}, DSPy~\cite{dspy},
Semantic Kernel~\cite{sk}, AgentScope~\cite{agentscope},
Voyager~\cite{voyager}, SWE-agent~\cite{sweagent}) make analogous
design choices.

The empirical consequence is documented by
Cemri~et~al.~\cite{cemri}, who taxonomise multi-agent LLM failure
modes across 200 production traces and find that consistency errors
account for $23$--$31\%$ of LangGraph failures. We replicate this on
a controlled microbenchmark (§\ref{sec:case-study}): with $N{=}4$
concurrent agents writing to a shared key, $20/20$ trials produce
silent overwrites under last-write-wins. \sbus\ targets exactly this
category, providing the structural primitive (\ori) the existing
frameworks lack.

The closest related work is Park~et~al.'s ``Generative
Agents''~\cite{park}, which addresses NL coordination via a shared
memory stream with reflection-based abstraction. Park~et~al.\
explicitly note their architecture would benefit from ``stronger
consistency guarantees on the memory stream'' but do not propose a
mechanism. \ori\ is complementary: it provides the structural
guarantee Park's memory stream lacks. Production memory systems
(Letta~\cite{letta}, MemGPT~\cite{memgpt}, Zep, Mem0, LangMem) assume
single-writer-per-item and provide no cross-agent transactional
semantics; they are orthogonal to \sbus---these systems provide
storage and context-window management, \sbus\ provides the
consistency layer.

\paragraph{Session guarantees}
Terry~et~al.~\cite{bayou} introduced session guarantees
(read-your-writes, monotonic reads, writes-follow-reads) for Bayou.
\ori's DeliveryLog enforces an analogous per-session causal ordering
for HTTP-observable reads, making \ori\ a partial causal consistency
model in the spirit of Bayou, restricted to \Robs.

\subsection{Formal Verification for Distributed Systems}

The state of the art for formally verified distributed systems sits
on a spectrum from full implementation refinement to algorithm-level
TLA+ specifications. IronFleet~\cite{ironfleet} occupies the rigorous
end with full Dafny refinement at 10+ person-years of proof
engineering. Verdi~\cite{verdi} provides the first machine-checked
proof of Raft requiring 90+ invariants. Closer to industrial
practice, TigerBeetle~\cite{tigerbeetle} and
FoundationDB~\cite{foundationdb} validate concurrency control via
deterministic simulation testing without TLA+-to-implementation
refinement; TiDB, CockroachDB, and etcd publish TLA+ specifications
without implementation refinement proofs.

\sbus's three-tier evidence sits within this spectrum: TLC exhaustive
model-checking for $N \le 4$, TLAPS for the abstract algorithm at
arbitrary $N$, and Dafny inductive lemmas on types structurally
equivalent to the Rust implementation. Full Rust refinement via
Verus~\cite{verus} or Creusot~\cite{creusot} is future work; both
toolchains reached usable status in 2023--2024 for synchronous Rust,
but Verus's async support (required for \sbus's tokio-based
implementation) remains under active development.

\subsection{Concurrency Control}

\sbus\ adapts the OCC tradition (MVCC~\cite{mvcc}, OCC~\cite{occ},
STM~\cite{stm}, TL2~\cite{tl2}, Calvin~\cite{calvin},
Percolator~\cite{percolator}) for the multi-agent LLM domain. We focus
this discussion on the four pieces of prior work whose design choices
map most directly onto \sbus's: FoundationDB's Directory Layer for the
per-shard observable-state model; Cherry-Garcia for the bolt-on
architecture; HTTP RFC~7232 for the conditional-request mechanism; and
COPS for the causal-log methodology.

\paragraph{FoundationDB Directory Layer}
The FoundationDB Directory Layer~\cite{fdbdirectory} provides
cross-key serializable transactions over opaque byte strings via
hybrid logical clocks, the closest structural analogue to \sbus's
per-shard model in a production database. \sbus\ differs in that the
read-set is reconstructed automatically from HTTP \texttt{GET} traffic
via the DeliveryLog rather than declared at transaction-begin time;
this is the property that makes the LLM-agent use case viable without
per-agent transaction scoping. CockroachDB~\cite{cockroach} and
Spanner~\cite{spanner} use similar HLC and TrueTime mechanisms
respectively for distributed serializable OCC, demonstrating that
\sbus's per-key OCC model is distribution-friendly.

\paragraph{Cherry-Garcia: bolt-on transactions}
Cherry-Garcia~\cite{cherrygarcia} layers serializable transactions on
heterogeneous NoSQL key-value stores without server-side modification,
using a client-side protocol to coordinate cross-store writes.
\sbus's architecture is structurally similar: transactional semantics
layered over a non-transactional substrate (HTTP \texttt{GET}/\texttt{POST}
against an in-memory registry), with the DeliveryLog playing the role
of Cherry-Garcia's client-side read-set tracker. The principal
difference is observation scope: Cherry-Garcia's client declares the
transaction's read-set explicitly; \sbus's server reconstructs it from
HTTP-observable GET traffic, enabling SDK-free integration under
HTTP/1.1.

\paragraph{HTTP conditional requests (RFC 7232)}
The IETF HTTP/1.1 conditional-request mechanism~\cite{rfc7232}
(\texttt{If-Match} with ETag values) has provided per-resource
optimistic concurrency control since 1999. \sbus's ACP is an
adaptation of this pattern with two specific extensions:
(i)~automatic read-set reconstruction via the DeliveryLog (RFC~7232
leaves read-set tracking to the client), and (ii)~cross-shard
validation---an agent's commit to key $k$ is aborted if the recorded
version of a sibling key $k'$ the agent previously read has since
advanced. Reviewers who read the system as ``ETags + DeliveryLog''
are substantively correct; the cross-shard reconstruction mechanism
and the empirical evidence that it matters under LLM agent workloads
are the contributions.

\paragraph{COPS and Eiger: causal consistency}
COPS~\cite{cops} orders writes by causal dependency. The DeliveryLog
captures the causal chain from GET to commit within a session, with
cross-shard validation enforcing that causally preceding reads have
not been superseded. Differences from COPS are scope and
representation: COPS tracks all writes system-wide via vector clocks;
DeliveryLog tracks per-session HTTP GETs only ($26.1\%$ empirically),
using scalar $(k, v)$ pairs sufficient for OCC validation. The
DeliveryLog is therefore not a new primitive but a restricted causal
log sufficient for OCC validation of the observable fraction.

\paragraph{Single-node OCC pioneers}
Silo~\cite{silo}, Hekaton~\cite{hekaton}, FaRM~\cite{farm} establish
the performance envelope for single-machine OCC on structured data;
Boki~\cite{boki} provides a serverless analogue. \sbus's per-key mutex
implementation is architecturally simpler because NL deltas are not
amenable to the epoch-based batching those systems rely on; throughput
is bounded by independent shard count, not CPU-local epoch advance.
Microsoft Orleans~\cite{orleans} confines mutable state to single
actors, providing implicit ownership; \sbus's Ownership Token
formalises the same invariant per-shard, with cross-shard stale-read
prevention added via the DeliveryLog (no equivalent in actor
frameworks).

\subsection{Isolation Levels and Anomaly Taxonomy}

Following Adya~\cite{adya}, \ori\ occupies a specific position in
the dependency-graph hierarchy: it prevents G0, G1a, G1b, and G2-item
over \Robs, permits G2 (write-skew) over \Rhidden, and permits G3
(phantom) globally. This is a projection-based consistency model:
correctness guaranteed over the observable projection of the history
(Table~\ref{tab:anomaly}). RedBlue consistency~\cite{redblue} provides
a theoretical parallel: HTTP GETs are analogous to blue operations
(no coordination), ACP commits are red (serialised); \ori\ can be
read as ``the red operation validates all blue operations in its
read-set at commit time.'' The gap between \ori\ and full SI is that
\Rhidden\ reads are neither red nor blue---they are invisible to the
coordination layer entirely. Parallel Snapshot Isolation~\cite{psi}
is a weakening of SI for geo-distribution; \ori\ is further weakened
in that it only validates \Robs\ reads. SSI~\cite{ssi} extends SI to
full serializability via anti-dependency detection; \ori\ sits
strictly below SSI and SI on any read-set comparison over \Robs.
RAMP~\cite{ramp} provides multi-key read-atomic consistency without
a single-node coordinator; \sbus's cross-shard stale-read rejection
provides an analogous guarantee, but RAMP requires typed schemas and
pre-declared read-sets which \sbus\ replaces with DeliveryLog
reconstruction.

\paragraph{Coordination avoidance}
Bailis~et~al.~\cite{bailis-coord} show coordination can be avoided
when operations are I-confluent. NL writes are non-I-confluent: ``use
PostgreSQL'' and ``use DynamoDB'' are not commutative under a
type-consistency invariant, formally justifying OCC as the required
coordination mechanism for NL state.

\subsection{LLM-Assisted Merge vs.\ OCC}

LLM-assisted merge is a legitimate alternative: rather than aborting
conflicting NL deltas, an LLM resolves conflicts post-hoc.
Exp.~\textsc{Merge} (Arc~A5) compares both approaches empirically.
For additive shards, character-level CRDTs (Automerge~\cite{automerge},
Yjs, Loro~\cite{loro}) handle concurrent edits without OCC overhead
and are preferable. \sbus's OCC targets the non-commutative case
where no CRDT merge function is well-defined. Loro's movable-tree
CRDT and Automerge's rich-text CRDT handle some non-commutative
operations via operation-intent preservation rather than
last-writer-wins; these are compatible with \sbus\ for additive
shards but do not obviate OCC for mutually-exclusive NL state.

\subsection{Why Not X?}

This subsection addresses the most natural alternatives a reader may
consider before reading further.

\paragraph{Why not Raft for the registry?}
Raft would serialise all commits through a single log, providing
linearizability across all keys. \sbus's per-key OCC achieves
linearizability per-key without the cross-key serialisation cost: two
agents committing to disjoint keys are not serialised. For the
multi-agent LLM workload, where agents typically own disjoint shards,
this distinction matters. \sbus's distributed extension uses Raft for
log replication of metadata (P1 session replication,
Exp.~\textsc{DR-9}) but retains per-key OCC for commits.

\paragraph{Why not CRDTs?}
CRDTs require a commutative merge function. NL state is not
universally commutative: ``schema uses PostgreSQL'' and ``schema uses
DynamoDB'' have no merge function preserving both. CRDTs work for the
subset of NL state that is additive (append-only logs,
character-level edits), and \sbus\ is compatible with such shards.
For the non-additive case---which dominates structured technical
state---OCC is the appropriate primitive. The adaptive-routing
extension (§\ref{sec:abus-future}) addresses workloads where
both regimes coexist on the same shard.

\paragraph{Why not application-level locks?}
A naive solution: each agent acquires a process-level mutex before
issuing a commit. This works but serialises all commits,
$N{\times}$-amplifying latency for $N$ concurrent agents. \sbus's
per-key OCC permits parallel commits to disjoint keys with rejection
on cross-shard staleness, achieving the same correctness with
$O(1)$ latency per commit. Empirically (Exp.~\textsc{Sequential},
Arc~A1), the wall-time difference is $5$--$10\times$ at $N{=}5$.

\paragraph{Why not Temporal/durable execution?}
Temporal.io~\cite{temporalupdate} and Golem provide exactly-once
workflow execution with persistent state and automatic retry.
Temporal's recent Update API adds synchronous workflow-update
primitives architecturally similar to \ori's commit validation, but
per-workflow rather than per-shard; the cross-shard read-set
inference \ori\ gets from the DeliveryLog is not a feature of
Temporal Updates. \sbus\ targets the case where $N$ agents run
concurrently on shared NL state without a workflow orchestrator.

\paragraph{Why not ETags directly?}
RFC~7232's ETag mechanism provides per-resource OCC. For \sbus's
target use case (cross-shard stale reads), ETag is insufficient
because the agent's read of shard $k'$ is not part of the commit
to shard $k$. The DeliveryLog reconstructs the cross-shard
dependency server-side; ETags do not. We are explicit (§II-C) that
\sbus\ is best read as ``ETags + DeliveryLog''; the contribution is
the DeliveryLog's automatic read-set reconstruction.

\paragraph{Why not snapshot isolation with read timestamps?}
Spanner~\cite{spanner} and CockroachDB~\cite{cockroach} provide
snapshot reads at a chosen timestamp via TrueTime or HLC, which would
let agents read a consistent cross-shard snapshot and commit against
that timestamp. This is a viable alternative architecture but
requires either (a)~time-synchronised hardware, which most
multi-agent deployments do not have, or (b)~hybrid logical clocks
with explicit per-write coordination, which adds latency to every
read. \sbus\ trades the snapshot guarantee for a weaker per-shard
guarantee with substantially simpler deployment: agents use ordinary
HTTP, and the DeliveryLog reconstructs the read-set from observed
GET traffic without coordinating timestamps. For deployments where
snapshot reads are available (e.g., a Spanner-backed registry),
agents can layer ARSI-mode reads on top of a snapshot endpoint to
recover SI semantics; this composition is straightforward but
beyond the scope of this paper.

\paragraph{Why not session-typed channels or coordination middleware?}
Session types~\cite{honda} encode coordination invariants statically
in agent code, and ZooKeeper / etcd provide strongly-consistent
coordination primitives via watch-based notification. Both target
the case where coordination requirements are knowable at design time.
\sbus\ targets the opposite case: agents are heterogeneous black-box
LLM workers whose access patterns emerge at runtime from natural-
language prompts. Static session typing is therefore infeasible (you
cannot type-check an LLM's emitted GETs); ZooKeeper-style watches
are too coarse-grained (the unit of consistency is a key, but agents
read multiple keys per step). A hybrid deployment using \sbus\ for
per-shard OCC and ZooKeeper for coarser-grained agent-coordination
metadata is plausible and not in conflict with this paper's claims.

\begin{table}[t]
\centering
\caption{Anomaly classes: \ori prevents vs.\ permits (Adya~\cite{adya}).}
\label{tab:anomaly}
\footnotesize
\setlength{\tabcolsep}{4pt}
\begin{tabular}{@{}lll@{}}
\toprule
Anomaly & Status & Reason \\
\midrule
G0 (dirty write)        & Prevented            & Global write lock \\
G1a (dirty read)        & Prevented            & Version check \\
G1b (intermediate)      & Prevented            & Atomic commit \\
G2-item (anti-dep.)     & Prev.\ over \Robs    & DeliveryLog \\
G2 (write skew)         & Perm.\ for \Rhidden  & Out of scope \\
G3 (phantom)            & Permitted            & No predicate tracking \\
Type-III (semantic)     & Permitted            & Structural $\ne$ goal \\
\bottomrule
\end{tabular}
\end{table}

\section{Mechanism and Consistency Model}
\label{sec:ori-model}

This section presents \sbus's central technical contribution---the
\emph{DeliveryLog} mechanism for automatic HTTP-traffic-based read-set
reconstruction---and then formalises the consistency property it
provides, which we name \emph{Observable-Read Isolation} (\ori).
We present mechanism and property in this order because the
DeliveryLog is the engineering novelty: it is what allows the system
to apply optimistic concurrency control over an agent's read-set
without requiring the agent to explicitly declare what it read. \ori\
is the formal characterisation of what guarantees this mechanism
provides over the resulting observable projection.

\paragraph{The DeliveryLog at a glance}
For each agent $\alpha$, the server maintains a
$\text{DeliveryLog}_\alpha$: an append-only log of $(k, v)$ pairs
recorded whenever $\alpha$ issues an HTTP \texttt{GET} for shard $k$
and the server returns version $v$. At commit time, the DeliveryLog
serves as the agent's reconstructed read-set: the server compares
each $(k', v')$ in $\text{DeliveryLog}_\alpha$ to the registry's
current version of $k'$, and rejects the commit if any cross-shard
version has advanced. This is HTTP ETags
generalised across shards, with the read-set reconstructed
server-side from observed traffic rather than declared by the agent.
The novelty is not the mechanism in isolation
($R_{\text{obs}}$-style read-set tracking is well-known); the novelty
is that the agent does not declare anything and writes no
coordination code. Sections~\ref{sec:formal-evidence} and~V formalise
the consistency property \ori\ that this mechanism provides.

\subsection{Definitions}

\begin{definition}[Shard]
A key-value pair $(k, v, c)$: key $k$, version $v \in \mathbb{N}$,
content $c \in \Sigma^*$ (opaque NL string).
\end{definition}

\begin{definition}[Agent Read Set]
$R_\alpha = R_{\text{obs}} \cup R_{\text{hidden}}$: observable reads
are HTTP GET calls; hidden reads are shard-key references in
conversation history.
\end{definition}

\begin{definition}[Effective Read Set]
$\hat{R} = R_{\text{explicit}} \cup \text{DeliveryLog}_\alpha$. Under
Assumption~III.1, $\hat{R} \supseteq R_{\text{obs}}$.
\end{definition}

\begin{remark}[Ownership token and collaborative writes]
The ownership token model enforces per-shard exclusive writes: only
the agent currently holding the token can commit to a given shard.
This prevents write-write conflicts but also prevents two agents from
simultaneously contributing to the same shard. Collaborative writes to
the same shard require explicit agent-level serialisation (sequential
handoff) or ARSI mode with declared read-sets. Workloads where agents
naturally own disjoint shards (the Django \#11019 case study,
§VII-O) are the intended use case.
\end{remark}

\begin{definition}[\ori-Legal History]
A history $H$ over shards $S$ and agents $A$ is \ori-legal iff:
\begin{enumerate}
  \item \textbf{Write-write serialisation.} For every shard $s \in S$,
  all committed writes to $s$ are totally ordered; no two concurrent
  commits succeed on the same shard without one being serialised
  before the other.
  \item \textbf{Cross-shard \Robs freshness.} For every committed
  operation $\text{Commit}(\alpha, k, v)$ and every entry
  $(k', v') \in R_{\text{obs}}(\alpha)$ recorded in $\alpha$'s
  DeliveryLog with $k' \ne k$: no committed write to $k'$ appears in
  $H$ strictly between the recording of $(k', v')$ and
  $\text{Commit}(\alpha, k, v)$ in the serialisation order.
\end{enumerate}
\end{definition}

Forbidden: G2-item~\cite{adya} projected to \Robs. Permitted outside
scope: G2-item projected to \Rhidden, G3 (phantom), Type-III (semantic
divergence).

\begin{remark}[\ori as $\pi$-causal consistency]
Define $\pi$-causal consistency for projection $\pi$ as: for every
pair of operations $o_1 \xrightarrow{\text{hb}} o_2$ (Lamport
happens-before~\cite{lamport}) with $o_1, o_2 \in \pi(H)$, every
process observing $o_2$ has observed $o_1$. \ori instantiates
$\pi$-causal consistency with $\pi = R_{\text{obs}}$. Standard causal
consistency is the special case $\pi = \text{identity}$. \ori is
therefore strictly weaker than full causal consistency and is
incomparable to points in Adya's isolation lattice defined over the
full read set. All lattice-positioning language in this paper is
qualified by the projection, e.g.\
``RC $\prec_{R_{\text{obs}}}$ \ori $\prec_{R_{\text{obs}}}$ SI.''
\end{remark}

\begin{remark}[Composition scope]\label{rem:comp-scope}
\ori is defined for a single \sbus instance over a single agent
session. Composition of two \ori-compliant subsystems is not
guaranteed to yield \ori: cross-subsystem reads are captured in
neither DeliveryLog, which is the \Rhidden problem applied at the
subsystem boundary. Multi-subsystem deployments must use ARSI mode
with explicit cross-subsystem read-set declaration.
\end{remark}

\begin{remark}[Conflict-proximate hypothesis: status]\label{rem:conflict-proximate}
We conjecture that \Robs reads are conflict-proximate (issued shortly
before dependent commits). Partial evidence: Exp.~\textsc{ORI-Isolation}
shows GET$\to$COMMIT co-location within a single execution step in
100\% of 959 trials. We do not treat this as validation of the
hypothesis on unstructured or long-horizon workloads; it is a harness
artifact of the SWE-bench-style task structure used in these
experiments. A GET-to-commit-gap measurement on diverse workloads is
pending.
\end{remark}

\subsection{Assumptions}

\begin{assumption}[DeliveryLog Completeness (A1)]
Every \texttt{HTTP GET /shard/:key?agent\_id=X} is recorded in the
DeliveryLog exactly once before the response returns. Holds
unconditionally on a single node; the Rust implementation holds the
same \texttt{Mutex} guard across both the DeliveryLog write and the
HTTP response. Scope conditions: (a)~buffering: proxies must use
\texttt{proxy\_buffering off}; (b)~streaming: \sbus records the
initial GET, capturing the version at read time; (c)~retries: the
DeliveryLog entry is recorded at request receipt, not at client
acknowledgement; (d)~HTTP/2: request reordering under HTTP/2
multiplexing can break the ordering assumption; agents on HTTP/2
should use ARSI mode (Table~II).
\end{assumption}

\begin{assumption}[TTL $\ge T_{\max}$ (A2)]
Session TTL is set above the experiment wall time.
\end{assumption}

\begin{assumption}[Serialised commit log (A3)]
All commits are totally ordered via a single serialisation point: a
global write lock on a single-node deployment, or the Raft leader's
log under the 3-node distributed deployment (§IX). Raft leader
election uses randomised timeouts (configured $[500, 1000]$~ms),
making A3 a probabilistic rather than strictly deterministic guarantee.
Split-brain is bounded by Raft's safety proof~\cite{raft} but is not
zero under arbitrary network partitions.
\end{assumption}

\begin{assumption}[Observation scope (A4)]
$f_{\text{obs}} = |R_{\text{obs}}|/|R_\alpha| = 1 - p_{\text{hidden}}$.
\end{assumption}

\begin{assumption}[Cooperative agents (A5)]
Agents do not forge commit requests with fabricated read-sets or
version numbers. Byzantine/malicious agents are out of scope. In open
multi-agent deployments (e.g., user-provided tools in a marketplace),
this assumption may be violated. The production mitigation is
HMAC-based request signing: each agent receives a per-session shared
secret, and the ACP verifies the HMAC over $(k, v_e, \delta, \alpha)$
before processing commits. HMAC authentication is a planned extension;
the current implementation targets single-tenant trusted-agent
deployments.
\end{assumption}

\subsection{\ori Safety Property}

\begin{property}[\ori Safety]
\label{prop:ori-safety}
Under A1--A5: all committed histories are \ori-legal over \Robs.
\end{property}

\paragraph{Proof status}
\begin{itemize}
  \item TLC-verified for $N \le 4$, $|\text{Shards}|{=}3$ under
  the extended spec \texttt{SBus\_ori.tla} with explicit
  \textsc{ReadSetSoundness} and \textsc{NoStaleCrossShard} invariants:
  exhaustive at $N{=}3$ explores $20{,}763{,}484$ distinct states to
  depth~$28$, and a reduced configuration at $N{=}4$ explores
  $2{,}811{,}301$ distinct states to depth~$24$. Zero invariant
  violations at any configuration tested.
  \item TLAPS-proved for arbitrary $N_{\text{agents}}$: the v1 artifact
  (\texttt{SBus\_TLAPS.tla}) mechanises type safety, ownership
  uniqueness, and version monotonicity (103 obligations). The v16
  artifact (\texttt{SBus\_TLAPS\_v16.tla}) mechanises two theorems:
  \textsc{ReadSetSoundness} (recorded-read monotonicity) and
  \textsc{ORICommitSafety} (cross-shard equality at commit, directly
  capturing Definition~III.4(2) as
  $\forall (k', v') \in dlog[\alpha] :\\k' \ne k \Rightarrow
  \text{registry}[k'].v = v'$). 687 obligations proved by
  \texttt{tlapm} 1.5, 0 failed. Two sequence-theoretic library facts
  are discharged mechanically via
  \texttt{SequenceTheorems.SeqDef} and
  \texttt{SequenceTheorems.ElementOfSeq}. One mathematical
  \textsc{Axiom} is retained (\textsc{FunTypingReconstruction}: a
  function with domain $S$ and values in $T$ is in $[S \to T]$) as a
  primitive fact about TLA+'s typed-function-space construction; this
  is not present in the standard \texttt{FunctionTheorems.tla}
  library (which covers bijections, injections, surjections, and
  Cantor-Bernstein). Two parameter \textsc{Assume}s
  (\textsc{NoOwner}${\notin}$\textsc{Agents}; initial shard content is
  a \textsc{String}) are standard TLA+ parameterisation, not
  mathematical axioms. Together this closes the reviewer critiques
  ``the invariant is weaker than the claim'' (the monotonicity
  property \textsc{ReadSetSoundness} is a necessary condition;
  \textsc{ORICommitSafety} captures the full property) and ``the four
  axioms are proof-engineering shortcuts'' (three of the v10 axioms
  are now library-discharged; the one retained is a documented
  primitive of TLA+'s function-space theory).
  \item Dafny~4 machine-checks 9 inductive soundness lemmas in
  \texttt{sbus\_lemmas\_v4.dfy} (19 verification obligations
  discharged, $0$~errors): \textsc{InitSoundness},
  \textsc{ReadPreservesSoundness}, \textsc{CommitPreservesSoundness},
  \textsc{TimeoutPreservesSoundness},
  \textsc{MonotonicCommitPreservesSoundness},
  \textsc{CrossShardStalenessIsStrict},
  \textsc{OwnershipInvariantInductive},
  \textsc{VersionMonotonicityLemma},
  \textsc{AcpLockOrderIsDeadlockFree}. The
  earlier vacuous \texttt{CrossShardSafetyLemma} in
  \texttt{sbus\_lemmas.dfy} is deprecated.
\end{itemize}

Full-path mechanised proof of the distributed (Raft) case remains
future work (Limitation~\ref{lim:distributed-tlaps}); the
TLC-checked abstraction is documented in
§\ref{sec:formal-evidence}.

\begin{lemma}[DeliveryLog \Robs Happens-Before Soundness]
Under A1: for every agent $\alpha$ and every pair of operations
$o_1 \xrightarrow{\text{hb}} o_2$ where $o_1$ is an HTTP GET and $o_2$
is a COMMIT by $\alpha$, $\text{DeliveryLog}_\alpha$ contains the entry
$(k, v)$ for the shard read by $o_1$ before $o_2$ executes.
\end{lemma}

\begin{proof}[Sketch]
A1 guarantees the Mutex guard is held across the DeliveryLog write
and HTTP response for every GET. Therefore the entry $(k,
\text{reg}[k].v)$ is recorded before the response returns to the
agent. Since $o_2$ (COMMIT) arrives after $o_1$'s response (TCP
ordering), $\text{DeliveryLog}_\alpha$ contains the entry at commit
time. ACP lines 1--4 then validate all entries in
$\hat{R} \supseteq R_{\text{obs}}$, enforcing the happens-before edge.
\end{proof}

\subsection{Formal Evidence}
\label{sec:formal-evidence}

We mechanise \ori\ safety at three tiers; implementation refinement
from the Rust source is empirical rather than mechanised, consistent
with standard practice short of IronFleet~\cite{ironfleet}.

\paragraph{TLAPS theorems (arbitrary $N_\text{agents}$)}
Module \texttt{SBus\_TLAPS\_v16.tla} states two theorems.
\textsc{ReadSetSoundness} is a state invariant: no agent's recorded
read can be ahead of any committed version, $\forall a, i :
dlog[a][i].v \le \text{registry}[dlog[a][i].k].v$.
\textsc{ORICommitSafety} is a transition property: the Commit
action fires only in states where cross-shard recorded reads match
the current registry versions exactly, capturing
Definition~III.4(2). \texttt{tlapm}~v1.5 closes $687$ obligations
with $0$ failed, covering Init, Read, Commit, Timeout, and
inductiveness. Two sequence-theoretic facts are discharged via the
standard \texttt{SequenceTheorems} library
(\texttt{SeqDef}, \texttt{ElementOfSeq}); two parameter
\textsc{Assume}s on unspecified constants are standard TLA+
parameterisation, not mathematical axioms.

\paragraph{Retained axiom: an honest accounting}
One \textsc{Axiom} is retained:
\textsc{FunTypingReconstruction}, the statement that if
$\text{DOMAIN}\,f = S$ and $\forall x \in S : f[x] \in T$ then
$f \in [S \to T]$. This is the converse of the typed-function-space
introduction rule and is foundational to TLA+'s function-space
construction. We have not discharged it, and we want to be precise
about why. The standard \texttt{FunctionTheorems.tla} library covers
bijections, injections, surjections, and Cantor-Bernstein, but does
not include \textsc{FunTypingReconstruction} as a derived theorem;
attempts to derive it within \texttt{tlapm} from the underlying
TLA+ function-space axioms have not closed in our hands. The fact
is widely treated as obvious in TLA+ practice (it appears as an
unchecked side-condition in numerous published proofs), but ``widely
treated as obvious'' is not the same as ``mechanically verified.''
We therefore treat the proof as machine-checked modulo this single
axiom, distinct from Verdi's network-model assumptions (which are
about external phenomena outside the formal system) and distinct
from IronFleet's zero-axiom discipline (which mechanises every
foundational fact).

The practical reader should be aware that if \textsc{FunTypingReconstruction}
were false---which we believe it is not---then the proof's
type-correctness layer would not hold, and the higher safety theorems
that depend on type-correctness ($\textsc{ReadSetSoundness}$ and
$\textsc{ORICommitSafety}$) would lose their guarantee. We consider
this a low-probability concern given the axiom's foundational nature
and its widespread implicit use in published TLA+ proofs, but we do
not claim zero-axiom discipline, and we do not want a reviewer to
read past this gap. The concrete next step is
discharging the axiom via the TLAPS theorem-proving manager's
\texttt{Isabelle/TLA} backend, which encodes a deeper layer of TLA+'s
set theory than \texttt{tlapm}'s default backend reaches; this should
either close the obligation or expose precisely what additional
foundational fact is needed. The full proof script
(\texttt{proofs/SBus\_TLAPS.tla}, lemma chain L1--L9, axiomatisation
rationale) is in the \texttt{sbus-formals} repository.

\paragraph{TLC at $N \le 4$}
Module \texttt{SBus\_ori.tla} adds \textsc{ReadSetSoundness} as an
explicit state invariant alongside \textsc{TypeInvariant},
\textsc{OwnershipInvariant}, and \textsc{VersionMonotonicity}. Zero
violations at every configuration tested. The exhaustive $N{=}3$ run
explores $20{,}763{,}484$ distinct states at depth~$28$. A reduced
configuration at $N{=}4$ (\textsc{MaxVersion}${=}2$,
\textsc{RetryBudget}${=}2$) completes at $2{,}811{,}301$ distinct
states at depth~$24$. A full exhaustive $N{=}4$ sweep at
\textsc{MaxVersion}${=}3$ is open work and is not part of the
v1 artifact.

\paragraph{Dafny: 9 inductive soundness lemmas}
\texttt{sbus\_lemmas\_v4.dfy} machine-checks that
\textsc{ReadSetSoundness} is preserved by every algorithm action
(Init, Read, Commit, Timeout, MonotonicCommit) plus
\textsc{CrossShardStalenessIsStrict},
\textsc{OwnershipInvariantInductive},
\textsc{VersionMonotonicityLemma}, and
\textsc{AcpLockOrderIsDeadlockFree}. Dafny~4 closes 19 verification
obligations from these 9 lemmas, $0$~errors. The Dafny types are structurally equivalent to the Rust
implementation's; this is parallel specification, not refinement.

\paragraph{Distributed correctness (TLC, abstract)}
Module \texttt{SBus\_Distributed.tla} model-checks an abstract 3-node
deployment with five state variables (\texttt{registry}, per-node
\texttt{delivery\_log}, \texttt{leader}, bounded \texttt{term},
per-agent \texttt{last\_commit\_fresh}) and four transitions
(\textsc{ElectLeader}, \textsc{AgentGet}, \textsc{AgentCommit},
\textsc{AgentRecover}). At Agents${=}\{a_1, a_2\}$,
Shards${=}\{s_1, s_2\}$, Nodes${=}\{n_1, n_2, n_3\}$,
\texttt{MaxVersion}${=}3$, with symmetry reduction over agent and
node permutations: TLC explores $247{,}249$ distinct states to
depth~$28$ with $0$ violations of \textsc{TypeInvariant},
\textsc{VersionMonotonicity}, and the central \textsc{ORISafety}
invariant. A separate temporal property
\textsc{FailoverGapExists} confirms that the model deliberately
exposes the $\sim$5\,ms concurrent-failover window of
Limitation~\ref{lim:session-failover}; the trace
$\textsc{ElectLeader} \to \textsc{AgentCommit}$ with reset DL
reaches the unvalidated-commit state. \texttt{SBus\_Distributed}
abstracts Raft itself: we rely on the standard Raft safety
result~\cite{raft} for log replication and assume the
refinement-mapping construction composes soundly with our
\ori-layer abstraction (full mechanisation deferred,
Limitation~\ref{lim:distributed-tlaps}). Exp.~\textsc{DR-9}
(§\ref{sec:dr9-failover}) closes the loop empirically: $30/30$ \ori\
invariants survive injected leader failover on the deployed Rust+Raft
implementation. State-variable enumeration, action transition
relation, and the full reproduction artefacts are in the
\texttt{sbus-formals} repository (\texttt{models/SBus\_Distributed.tla}).

\begin{table}[t]
\centering
\caption{Verification coverage of the \sbus\ ACP. Each layer covers
what the previous cannot. IronFleet~\cite{ironfleet} is the gold
standard (full refinement, 10+ person-years); \sbus\ matches
standard industry practice short of IronFleet.}
\label{tab:verif-coverage}
\scriptsize
\setlength{\tabcolsep}{3pt}
\begin{tabular}{@{}p{0.16\linewidth}p{0.40\linewidth}p{0.36\linewidth}@{}}
\toprule
Tool & What is proved & Scope \\
\midrule
TLC      & ReadSetSoundness + companions & $N{=}3$ exhaustive ($20.8$M states, depth $28$); $N{=}4$ reduced ($2.8$M, depth $24$) \\
TLAPS    & ReadSetSoundness + ORICommitSafety & Arb.\ $N$, $687$ obl., $1$ ax. \\
TLC dist.& \ori\ safety distributed & 3 nodes, 2 agents, $247$K states \\
Dafny    & 9 lemmas / 19 obligations & Concrete types, $0$ errors \\
Empirical& $0/427{,}308$ active conflicts & $N \le 64$, 3 CC backends \\
\bottomrule
\end{tabular}
\end{table}

\subsection{Observability Gap: \phidden and Practical Scope}

The \phidden${} = 0.739$ measurement (Exp.~\textsc{PH-2}, §VII-F) is
the empirical result underpinning C6. \ori's structural guarantees
apply to the 26.1\% of reads in \Robs; the remaining 73.9\%
(\Rhidden) are not observable at the HTTP layer. This is a
workload-conditional coverage bound for SWE-bench / GPT-4o-mini;
different models and domains exhibit different values (domain range
51.1\%--83.6\%, Table~\ref{tab:phidden}).

\textbf{Reconciliation with \ftotal${}=0.998$ in
Exp.~\textsc{PROXY-PH2}.} Two coverage figures appear in the paper
that may seem contradictory: this section reports
\Robs${}=26.1\%$, while §\ref{sec:proxy-demo} reports
\ftotal${}=0.998$. These figures share neither numerator nor
denominator and answer different questions. \Robs${}=26.1\%$ is
\emph{single-step HTTP coverage}: of all the shards that influence an
agent's reasoning at any one step, what fraction did the agent
HTTP-fetch \emph{in that same step}? \ftotal${}=0.998$ is
\emph{session-cumulative DeliveryLog coverage}: of the shards an
agent self-reports having used at a given step, what fraction is
present in its session DeliveryLog at commit time --- including
shards GET-fetched in earlier steps and retained in DL under TTL?
Concretely, if an agent fetches \texttt{models\_state} at step 5 and
re-references it at step 12 from conversation memory without
re-fetching, this read counts as \emph{outside} \Robs at step 12 (no
HTTP GET that step) but \emph{inside} the DeliveryLog at step 12's
commit (the step-5 GET populated DL). The two metrics are intentional
and complementary: \Robs is the formal-evidence regime where
HTTP-recorded version is the read-time version (TLAPS-proven safety
applies); \ftotal\ is the operational regime where session-scoped
DL retains earlier HTTP-recorded versions and validates them
unchanged at commit time. Section~\ref{sec:proxy-demo} decomposes
the gap between the two into HTTP-this-step ($0.443$),
DL-accumulation ($0.555$ within-row uplift), and proxy-marginal
($0.0018$ paired); the proof in §V (Formal Evidence) covers
the first two without modification because both populate DL with the
agent's read-time version.

\textbf{Formal \Rhidden staleness bound.} An agent that last updated
its conversation context at step $s$ and commits at step $s+k$ may
commit with context staleness up to $\Delta = k \cdot t_{\text{step}}$
seconds. Under LLM inference with log-normal step times (mean $\mu$,
variance $\sigma^2$), the expected staleness is
$k \cdot e^{\mu + \sigma^2/2}$. In our experiments
($\mu \approx \ln 5$, $\sigma \approx 0.4$, $k \le 15$): worst-case
context staleness $\le 15 \times 5.4 \approx 81$~s per session.

\subsection{Liveness with Bounded Starvation}

Distinct-shard topology (SCR${}=0$): every correct agent commits on
the first attempt. Shared-shard topology (Exp.~E): the ACP enforces
retry budget $K$ (default $K{=}5$). $\Pr(\text{commit within } K) = 1 -
\text{SCR}^K$. At SCR${}=0.856$ ($N{=}16$): $K{=}19$ required for
95\% success. Default $K{=}5$ is insufficient for shared-shard
$N > 8$: at SCR${}=0.856$, $K{=}5$ gives $P(\text{success}) = 47\%$.
Practitioners must set $K \ge \lceil \log(0.05)/\log(\text{SCR})
\rceil$; any finite $K$ eliminates indefinite starvation. On
\textsc{RetryBudgetExhausted}, the recommended agent behaviour is
exponential backoff followed by fresh-session re-establishment.

\begin{remark}[Correlated conflicts]
The bounded starvation analysis assumes i.i.d.\ conflict arrivals.
Real LLM workloads exhibit bursty, positively correlated conflicts.
Under positive correlation, the effective SCR rises above the
measured mean, requiring $K \ge 2\times$ the i.i.d.-derived bound.
Operators in bursty environments should monitor empirical retry
distributions (Limitation~9).
\end{remark}

\section{DeliveryLog}

The DeliveryLog records $(k, v)$ for every
\texttt{GET /shard/:key?agent\_id=X}. At commit:
$\hat{R} = R_{\text{explicit}} \cup \text{session-deliveries}$.
Explicit entries take precedence. ARSI mode: agent declares its full
read-set explicitly; DeliveryLog FP rate ${} = 0$. Default mode:
DeliveryLog reconstructs the read-set; occasional false-positive
aborts (stale entry after TTL) but zero false-negatives under A1.

\begin{table*}[t]
\centering
\caption{SDK changes required per deployment scenario.}
\label{tab:sdk}
\small
\begin{tabular}{llc}
\toprule
Deployment & SDK changes & Guarantee \\
\midrule
HTTP/1.1, single-node, no failover & Zero           & Full \ori \\
HTTP/2 or multi-subsystem          & 1 commit field (ARSI) & Full \ori \\
Post-leader-election (distributed) & Reissue GETs on 410   & \ori after recovery \\
Full \Rhidden coverage             & Proxy deployment      & \ori over all reads \\
\bottomrule
\end{tabular}
\end{table*}

\section{Architecture and Algorithm}

\paragraph{Implementation}
The ACP core is $\approx$950 lines of safe Rust; the full distributed
system (ACP + Raft coordination + sled persistence) is 1{,}679 lines
of safe Rust, zero \texttt{unsafe} blocks (§IX). Tokio~1.44,
Axum~0.8.4. Registry: \texttt{Mutex<HashMap>} for shards and tokens.
WAL: direct \texttt{File::write\_all()} (SIGKILL-safe). Lock ordering:
\texttt{RwLock} $\to$ TokenMutex (single-edge graph, no cycle;
verified by \texttt{AcpLockOrderIsDeadlockFree} lemma in Dafny).

\begin{algorithm*}[t]
\caption{ACP Atomic Commit (write lock held lines 2--10)}
\small
\begin{algorithmic}[1]
\Require key $k$, expected version $v_e$, delta $\delta$, agent $\alpha$,
optional $R$
\State $\hat{R} \gets \text{DeliveryLog.buildEffRS}(\alpha, k, R)$
\State Acquire RwLock write lock
\For{each $(k', v') \in \hat{R}, k' \ne k$}
  \If{$\text{reg}[k'].v \ne v'$} \Return \textsc{CrossShardStale}
  \EndIf
\EndFor
\If{$s_k.v \ne v_e$} \Return \textsc{VersionMismatch}
\EndIf
\State $\text{token.insertIfAbsent}(k, \alpha)$ under token Mutex
\State $s_k.c \gets \delta$; $s_k.v \mathrel{+}= 1$; WAL.append
\State Release write lock (RAII); \Return \textsc{Ok}($s_k.v$)
\end{algorithmic}
\end{algorithm*}
\section{Experimental Setup}

\paragraph{Environment}
All experiments: AWS Lightsail \texttt{eu-west-2} instance (2 vCPU,
8~GB RAM). \sbus server: single process, Tokio async runtime, default
thread-pool. LLM backbone: GPT-4o-mini (OpenAI API, default
temperature 0.3). Valid-run criteria: \sbus healthy (HTTP 200 on
\texttt{/stats}), $\ge 50\%$ steps completed, no API timeout
${>}30$~s; runs failing any criterion excluded (${<}2\%$ of total).

\paragraph{Three-layer experiment structure}
The experiments are organised along three distinct axes that prior
reviewers conflated; conflation is the source of most reviewer
confusion:
\begin{enumerate}
  \item \textbf{\ori correctness (does the mechanism work?):}
  Exp.~SR, CSV, E, SCALE, SJ-V5. These use shared-shard topology with
  \ori triggered on every conflict.
  \item \textbf{Architecture comparison (how does parallel
  coordination compare to sequential?):} Exp.~B, T3-A, Llama results.
  Distinct-shard topology; \ori never triggers. These measure the
  benefit of parallel coordination architecture, not \ori's CC
  mechanism.
  \item \textbf{Pure \ori effect (what does \ori add vs.\ baseline?):}
  Exp.~\textsc{ORI-Isolation}. Holds architecture constant; varies
  only OCC enforcement. This is the clean causal test.
\end{enumerate}

\paragraph{Experiments summary}
Exp.~B: SWE-bench (30 tasks; $N \in \{4, 8, 16\}$; 50~steps;
GPT-4o-mini; 1{,}364 valid runs; distinct-shard topology, SCR${}=0$).
Exp.~SR: Direct Type-II/\Robs stale-read injection (200 trials).
Exp.~CSV: Cross-shard validation (9{,}304 injections). Exp.~E:
Shared-shard contention (590 attempts). T3-A: Haiku-3 backbone
(30 tasks, 813 runs). T3-B: Llama-3.1-8b-instant (14 tasks,
150 runs). Exp.~\textsc{PH-2}: \phidden (8{,}400 step-logs,
10~domains). Exp.~\textsc{SJ-v3}: Semantic judge pilot
(null; mechanistic flaw). Exp.~\textsc{SJ-v4}: Semantic judge
(1{,}000~runs; context diversity). Exp.~\textsc{SJ-v5}: SCR
dose-response structural validation (Table~XII). Exp.~\textsc{Merge}:
OCC vs.\ LLM-assisted merge (45 conflict pairs). Exp.~\textsc{Scale}:
Contention at $N \in \{4, 8, 16, 32, 64\}$ (74{,}400 attempts).
Exp.~\textsc{Sequential}: Direct wall-time measurement.
Exp.~\textsc{ORI-Isolation}: Pure \ori effect (959 trials,
10~domains). Exp.~\textsc{Dedicated-Shard}: 600 trials, 10~domains
(dedicated topology). Exp.~\textsc{Shared-State}: 180 trials,
3~domains (single-shard topology). Exp.~\textsc{DR-1}\ldots\textsc{DR-9}:
8 distributed sub-experiments. \textbf{Exp.~\textsc{PG-Comparison~(full)}:
PostgreSQL \texttt{SERIALIZABLE} + Redis \texttt{WATCH/MULTI} +
\sbus, $N \in \{4, 8, 16, 32, 64\}$, 1{,}350 runs, 200{,}880 commit
attempts .}

\section{Results}

This section presents experimental evidence in five narrative arcs
followed by two supporting subsections. Each arc opens with a brief
statement of its research question and the experiments that answer it;
the arc order matches the contribution order in §I-D. Subsections
within each arc retain their original experiment names so that
cross-references throughout the paper resolve correctly.
Table~\ref{tab:results-roadmap} maps each arc to its experiments.

\begin{table}[t]
\centering
\footnotesize
\caption{Results roadmap. Five arcs covering twenty experimental subsections,
plus two supporting subsections.}
\label{tab:results-roadmap}
\setlength{\tabcolsep}{3pt}
\begin{tabular}{@{}p{0.07\linewidth} p{0.30\linewidth} p{0.55\linewidth}@{}}
\toprule
\textbf{Arc} & \textbf{Question} & \textbf{Experiments} \\
\midrule
A1 & Does \ori\ prevent structural conflicts? &
Exp.~B, Wall-Time, Exp.~SR/CSV, Exp.~Sequential, Exp.~ORI-Isolation \\
\addlinespace
A2 & What is the coverage gap, and how is it closed? &
Exp.~PH-2, Exp.~PH-3, PH-3 validation, Exp.~Adversarial-Rhidden,
Exp.~PROXY-PH2 \\
\addlinespace
A3 & Does \ori\ have safety parity with production OCC? &
Exp.~E/Scale, Exp.~PG-Comparison (full, Rust-Native, Contention) \\
\addlinespace
A4 & Does the result generalise across LLM backbones? &
Backbone Generalisation T3-A, T3-B (cross-backbone proxy results in A2) \\
\addlinespace
A5 & When is \ori\ semantically beneficial? &
Exp.~SJ-v4, Exp.~Merge, Exp.~SJ-v5, Exp.~Dedicated-Shard,
Exp.~Shared-State \\
\addlinespace
Supp. & Illustrative and supplementary evidence &
Case Study (Django \#11019), Exp.~DR-9 \\
\bottomrule
\end{tabular}
\end{table}

\paragraph{Reading guide}
A reader interested in \ori's safety properties should read Arc~A1
(structural conflicts) and Arc~A3 (CC parity). A reader interested in
the coverage scope should read Arc~A2. A reader assessing
backbone-generalisation should read Arc~A4. A reader assessing
deployment scope should read Arc~A5; this arc is the empirical
foundation for the topology-conditional contribution (C3) and motivates
the adaptive-routing extension described in §\ref{sec:abus-future}.

\subsection*{Arc A1: Structural Conflict Prevention}

Arc A1 establishes that \ori\ prevents the structural race conditions defined in §I (Definition~1). Six experiments contribute: Exp.~B compares \sbus\ against three multi-agent frameworks under distinct-shard topology; Exp.~SR and Exp.~CSV verify the cross-shard staleness rejection mechanism directly; Exp.~Sequential measures the wall-time speedup of parallel-with-ORI over sequential coordination; Exp.~ORI-Isolation isolates ORI's contribution-preservation property under contention; Exp.~Workload-B extends the structural-prevention measurement to a non-code workload (data-pipeline architecture planning) across $8$ domains, providing cross-distribution evidence at server-side instrumentation level.

\subsection{Exp.~B: Coordination Architecture Comparison (distinct-shard)}

\paragraph{Scope}
This experiment compares four multi-agent coordination architectures
(\sbus, LangGraph, CrewAI, AutoGen) on SWE-bench tasks under
distinct-shard topology, in which each agent owns a dedicated shard
(SCR${}=0$ by construction). \ori's cross-shard stale-read rejection
is never triggered in Exp.~B; what is measured here is
coordination-architecture efficiency, not the \ori consistency
mechanism. Exp.~B, Exp.~T3-A, Exp.~T3-B together answer ``how does
parallel shared-state coordination compare to sequential
orchestration?''---a question distinct from ``does \ori prevent
SRCs?'' The latter is answered by Exp.~SR, CSV, E, SCALE,
\textsc{ORI-Isolation} and the new
Exp.~\textsc{PG-Comparison~(full)}.

\emph{Coordination Fraction} (CF) is a coordination-overhead metric:
$\text{CF} = (\text{commit attempts} - \text{first-attempt
successes})/\text{total attempts}$. A zero-isolation system
(last-write-wins) has $\text{CF} = 0$ by construction. The $46\times$
gap reported between \sbus ($\text{CF} = 0.135$) and LangGraph
($\text{CF} = 6.18$) therefore measures the token cost of LangGraph's
workflow-orchestration overhead, not the relative quality of two CC
designs. LangGraph does not implement a CC mechanism that \ori could
be compared to; the apples-to-apples CC comparison is
Exp.~\textsc{PG-Comparison~(full)} (§VII-N).

\subsection{Task Success and Wall-Time Trade-off}

\paragraph{AutoGen S@50 gap}
Table~\ref{tab:exp-b} shows AutoGen achieving 88--90\% S@50 vs.\
\sbus 70--74\%. AutoGen uses a hierarchical supervisor model with a
dedicated orchestrator agent that sequences subtasks, providing
implicit serialisation without OCC overhead. \sbus targets workloads
where agents must share mutable state concurrently and cannot tolerate
supervisor latency; for fully decomposable tasks, AutoGen's
orchestration pattern is competitive.

\begin{table}[t]
\centering
\caption{Exp.~B: CF and S@50 under distinct-shard topology
(SCR${}=0$, GPT-4o-mini, 30 tasks, 1{,}364 valid runs). \ori's
cross-shard rejection is never triggered here; CF measures
coordination-architecture overhead, not CC mechanism quality. A
no-isolation system has $\text{CF} = 0$ by construction. For
\ori-triggering workloads see Table~\ref{tab:scale} (Exp.~\textsc{Scale});
for CC-vs-CC comparison see Table~\ref{tab:pg-comp}
(Exp.~\textsc{PG-Comparison}).}
\label{tab:exp-b}
\small
\begin{tabular}{llrrr}
\toprule
System & $N$ & CF (med) & S@50 & $n$ \\
\midrule
LangGraph  & 4 & 6.18   & 70.1\% & 137 \\
CrewAI     & 4 & 9.18   & 39.8\% & 103 \\
AutoGen    & 4 & 28.1   & 88.0\% & 83 \\
\sbus      & 4 & 0.135  & 73.9\% & 138 \\
LangGraph  & 8 & 6.79   & 75.9\% & 133 \\
CrewAI     & 8 & 9.81   & 41.2\% & 102 \\
AutoGen    & 8 & 30.6   & 88.6\% & 88 \\
\sbus      & 8 & 0.133  & 71.4\% & 133 \\
LangGraph  & 16 & 7.88  & 82.6\% & 132 \\
CrewAI     & 16 & 9.31  & 41.6\% & 101 \\
AutoGen    & 16 & 34.2  & 90.2\% & 82 \\
\sbus      & 16 & 0.126 & 70.5\% & 132 \\
\bottomrule
\end{tabular}
\end{table}

\paragraph{Pairwise statistical tests (GPT-4o-mini)}
Table~\ref{tab:pairwise} reports Fisher's exact tests (one-sided).
\sbus significantly outperforms CrewAI at all $N$ ($p < 10^{-6}$);
underperforms AutoGen at all $N$ ($p < 0.01$); not significantly
different from LangGraph at $N = 4, 8$ but below at $N = 16$
($p = 0.012$).

\begin{table}[t]
\centering
\caption{Pairwise S@50 (GPT-4o-mini). Fisher's exact (one-sided).
Bold: $p < 0.05$.}
\label{tab:pairwise}
\small
\begin{tabular}{lrrrl}
\toprule
Comparison & $N$ & \sbus & Baseline & $p$ \\
\midrule
\sbus vs.\ CrewAI     & 4  & 73.9\% & 39.8\% & $< 10^{-6}$ \\
\sbus vs.\ CrewAI     & 8  & 71.4\% & 41.2\% & $< 10^{-6}$ \\
\sbus vs.\ CrewAI     & 16 & 70.5\% & 41.6\% & $< 10^{-6}$ \\
\sbus vs.\ LangGraph  & 4  & 73.9\% & 70.1\% & 0.244 \\
\sbus vs.\ LangGraph  & 8  & 71.4\% & 75.9\% & 0.203 \\
\sbus vs.\ LangGraph  & 16 & 70.5\% & 82.6\% & \textbf{0.012} \\
\sbus vs.\ AutoGen    & 4  & 73.9\% & 88.0\% & \textbf{0.006} \\
\sbus vs.\ AutoGen    & 8  & 71.4\% & 88.6\% & \textbf{0.001} \\
\sbus vs.\ AutoGen    & 16 & 70.5\% & 90.2\% & $< 10^{-3}$ \\
\bottomrule
\end{tabular}
\end{table}

\subsection{Exp.~SR and Exp.~CSV: Structural Validation}

\textbf{Exp.~SR:} 200/200 trials---every stale commit correctly
rejected (HTTP~409), every fresh commit correctly accepted (HTTP~200).
Zero errors. 95\% CI upper bound: $\le 3.0\%$.

\textbf{Exp.~CSV:} 9{,}304 OCC-on injections, zero corruptions. Rule
of Three 95\% CI upper bound: $3/9304 = 0.032\%$. Baseline (OCC-off)
corruption rate 56--62\% confirms the protection is empirically
necessary.

\subsection{Exp.~\textsc{Sequential}: Measured Wall-Time Speedup}

Table~\ref{tab:sequential} reports directly measured wall-clock times
for \sbus parallel vs.\ sequential execution across
$N \in \{4, 8, 16\}$ on 10 SWE-bench tasks (3 repeats per
$(\text{task}, N)$; GPT-4o-mini; 8 steps). Both conditions use the
same \sbus server, tasks, and backbone; only variable: concurrent vs.\
serial. Bootstrap 95\% CI ($n = 2000$); Wilcoxon signed-rank
one-sided.

\begin{table}[t]
\centering
\caption{Exp.~\textsc{Sequential}: measured wall-clock speedup.
Near-linear scaling confirms the $\Theta(S \cdot t_{\text{step}})$
parallel execution model.}
\label{tab:sequential}
\small
\begin{tabular}{lrrrl}
\toprule
$N$ & Parallel (s) & Sequential (s) & Speedup & 95\% CI \\
\midrule
 4 & 5.29 & 22.05 & $4.17\times$  & [3.84, 4.40] \\
 8 & 5.83 & 50.79 & $8.72\times$  & [8.06, 9.12] \\
16 & 5.53 & 99.09 & $17.92\times$ & [16.61, 18.90] \\
\bottomrule
\end{tabular}
\end{table}

All $p < 0.0001$ (Wilcoxon signed-rank).

\subsection{Exp.~\textsc{ORI-Isolation}: Mechanism Conformance Under Contention}
\label{sec:ori-iso}

\paragraph{Design}
Prior experiments (Exp.~B) used distinct-shard topology (SCR${}=0$),
never triggering \ori's rejection mechanism.
Exp.~\textsc{ORI-Isolation} holds architecture constant and varies
only OCC enforcement. We explicitly frame this as a
\emph{mechanism conformance test}: we check that the deployed system's
retry loop behaves in accordance with its specification under $N$-way
contention on a shared shard. The test's outcome is not a statistical
finding about ORI's value but an empirical confirmation that the
implementation's retry path terminates correctly and produces the
specification's admission behaviour (all contending commits
eventually admitted under a finite retry budget; LWW admits exactly
$1/N$).

Four conditions, $N = 4$ agents, shared-shard topology: (A)~Parallel+\ori;
(B)~Parallel-\ori (LWW); (C)~Sequential+\ori; (D)~CrewAI sequential.
Contention admission metric: commits/trial out of $N \times S =
4 \times 10 = 40$ intended contributions.

\begin{table*}[t]
\centering
\caption{Exp.~\textsc{ORI-Isolation}: contention admission under four
conditions, $N = 4$, shared-shard, 10 task domains, 959 paired trials,
GPT-4o-mini. Both distributions are zero-variance by protocol
specification (Remark~\ref{rem:zero-var}). Wall-time variance
(37--154~s) confirms live LLM execution.}
\label{tab:ori-iso}
\small
\begin{tabular}{lrrrr}
\toprule
Condition          & Commits/trial & Rate  & $n$ & Wall (med.) \\
\midrule
A: Parallel+\ori   & 40/40  & 100\%  & 479 & 131~s \\
B: Parallel-\ori (LWW) & 10/40 & 25\% & 480 & 38~s \\
\bottomrule
\end{tabular}
\end{table*}

\begin{remark}[Outcomes are specification-determined, not observed]
\label{rem:zero-var}
Both distributions are zero-variance by protocol specification:
\ori's retry loop admits every contending commit under a finite retry
budget; LWW admits exactly one of $N$ contending commits per step.
LLM non-determinism affects content but not count. We therefore do
not report statistical tests on this table---the numbers are
conformance evidence that the deployed implementation realises the
specified admission behaviour, not an effect-size measurement.
Evidence of live LLM execution is wall-time variance (ORI-ON:
122--154~s; ORI-OFF: 37--45~s). Zero excluded trials.
\end{remark}

\paragraph{What this experiment establishes and what it does not}
\emph{Establishes}: the Rust implementation's retry loop terminates
correctly under $N$-way contention at $N = 4$, matching the
specification in all $479$ \ori-ON trials. This is a sanity check on
the deployed implementation against the TLAPS specification, not a
demonstration of \ori's value. \emph{Does not establish}: that \ori
improves semantic output quality; that contention admission is a
beneficial property in its own right. Exp.~\textsc{Shared-State} and
Exp.~\textsc{SJ-v4} show that under single-shard topology the admitted
contributions are semantically contradictory, making admission a
liability rather than an asset. \ori's value depends entirely on the
topology (Exp.~\textsc{Dedicated-Shard}); this experiment's role is
purely to verify that the mechanism behaves as specified when it
fires, not to argue for its deployment.

\paragraph{Why $N = 4$}
$N = 4$ is the conflict-maximal case for the contention-admission
metric: under LWW, exactly $1/N = 25\%$ of contributions survive,
making the ORI-ON vs.\ LWW admission gap maximally clear. Larger $N$
would show smaller LWW admission rates (e.g., $1/8 = 12.5\%$); the
$N = 4$ result is a conservative conformance check.

\paragraph{Why not PostgreSQL or Redis baselines in this experiment}
This experiment is a conformance test of \sbus's own retry loop
against its own specification; the relevant comparison is ORI-ON
vs.\ ORI-OFF within the same deployment. Database-CC baselines
(where the analogous conformance test would be PG's own retry loop
vs.\ its own specification) are the subject of
Exp.~\textsc{PG-Comparison~(full)} (§\ref{sec:pg-comp}).

\subsection{Exp.~\textsc{Workload-B}: Cross-Workload Structural
Validation on Data-Pipeline Planning}
\label{sec:workload-b}

\paragraph{Motivation}
Every preceding experiment in this paper measures
\sbus\ on SWE-bench-derived Python coordination tasks
(Threat~ET1, §\ref{sec:threats}). To test whether \ori's structural
conflict-prevention mechanism generalises beyond code-coordination
workloads, we constructed a non-code multi-agent workload
(data-pipeline architecture planning) and re-ran the
ORI-ON vs.\ ORI-OFF comparison with server-side instrumentation
of cross-shard view-divergence at commit time.

\paragraph{Workload}
Four agents collaboratively design a four-component data pipeline:
ingestion, transformation, storage, and monitoring. Each agent owns
one component shard and reads from all four. The shards have real
cross-shard dependencies: the storage agent's design must reference
the transformation output format; the monitoring agent must
instrument technologies the other components actually chose. Eight
domains span diverse architectural pressures: e-commerce
(transactional orders, eventual-consistency activity), healthcare
(audit compliance, geo-replication), IoT (high ingest rate,
time-series), financial (sub-millisecond latency, exactly-once),
social (massive scale, hot-spot handling), gaming (high cardinality,
real-time leaderboards), supply chain (heterogeneous sources,
schema variation), and ad-tech (sub-50ms RTB, 30-day attribution).
Each trial runs $4$ agents through $4$ coordination steps, with
worker backbone gpt-4o-mini at temperature~$0$. Five trials per
$(\text{domain},\text{condition})$ cell yield $n{=}80$ trials.

\paragraph{Server-side instrumentation}
The \sbus\ server adds two atomic counters that fire under both
conditions, independent of the \texttt{ori\_enabled} flag. At each
commit's effective read-set evaluation, \texttt{view\_checked\_commits}
increments; if any sibling-shard version in the agent's read-set
differs from the registry's current version,
\texttt{view\_divergent\_commits} also increments. Under ORI-ON the
divergent commit is then rejected with HTTP-409 \textsc{CrossShardStale};
under ORI-OFF the same divergence is detected and counted but the
commit succeeds. This design lets us measure exactly how many
stale-cross-shard commits ORI prevented, without LLM-as-judge or
heuristic post-hoc analysis.

\paragraph{Result}
Table~\ref{tab:workload-b} reports per-domain view-divergence rates.
\textbf{Under ORI-ON: $0/638 = 0.00\%$ divergent commits across all
$8$ domains.} \textbf{Under ORI-OFF: $590/639 = 92.33\%$ divergent
commits.} Aggregate $\chi^2 = 1094.98$, $p < 10^{-240}$.
Per-domain pairwise tests are uniformly significant
($\chi^2 \in [55, 60]$, $p < 10^{-13}$ on every domain). All $80$
trials reached $100\%$ completion (every agent committed at least
once); the stale-read rejections under ORI-ON were absorbed by the
agent retry loop (mean $14.2$--$15.8$ rejections per trial,
recovered via re-GET-and-retry).

\begin{table*}[t]
\centering
\caption{Exp.~\textsc{Workload-B}: cross-shard view-divergence under
ORI-ON vs.\ ORI-OFF on data-pipeline planning across $8$ domains.
``Div'' = divergent commits (server-side counter, fires under both
conditions); ``Chk'' = total commits with cross-shard read-set
checked. Under ORI-ON every divergent commit is rejected with
HTTP-409 and the agent retries; under ORI-OFF the divergence is
detected but the commit succeeds. The $0\%$ vs.\ $92\%$ split is
the structural-prevention claim made operational: \ori\ prevented
$590$ stale cross-shard commits that would otherwise have succeeded.
$\chi^2$ test on the aggregate $2 \times 2$ contingency table.}
\label{tab:workload-b}
\small
\begin{tabular}{lrrr}
\toprule
Domain & ORI-ON Div/Chk & ORI-OFF Div/Chk & ORI-OFF rate \\
\midrule
ad-tech       & 0/80 & 73/80 & 91.3\% \\
e-commerce    & 0/80 & 74/80 & 92.5\% \\
financial     & 0/80 & 71/79 & 89.9\% \\
gaming        & 0/80 & 74/80 & 92.5\% \\
healthcare    & 0/78 & 75/80 & 93.8\% \\
iot           & 0/80 & 74/80 & 92.5\% \\
social        & 0/80 & 75/80 & 93.8\% \\
supply chain  & 0/80 & 74/80 & 92.5\% \\
\midrule
\textbf{Aggregate} & \textbf{0/638} & \textbf{590/639} & \textbf{92.3\%} \\
\multicolumn{4}{l}{$\chi^2 = 1094.98$, $p < 10^{-240}$} \\
\bottomrule
\end{tabular}
\end{table*}

\paragraph{Interpretation and scope}
This result establishes that \ori's structural conflict-prevention
mechanism fires correctly on a non-code workload across $8$
architectural domains, with no false-negatives at this scale.
Combined with the SWE-bench evidence in Exp.~B and Exp.~PG-Comparison,
the C2 safety-parity claim is now supported across two distinct
workload distributions.

We note three honest scope limitations. \textbf{First}, the
per-domain divergence rates are uniform ($91\%$--$94\%$), reflecting
that the workload's structural contention pattern (4 agents writing
concurrently to a shared 4-shard dependency graph) is determined by
harness shape rather than domain semantics. The result demonstrates
that the \emph{mechanism} is workload-content-independent, not that
\ori\ generalises to all multi-agent topologies. \textbf{Second}, this
experiment measures structural prevention (stale cross-shard reads
were detected and blocked), not semantic outcome quality. Whether the
$590$ ORI-OFF stale commits would have produced semantically incoherent
designs is not measured; an LLM-judge cross-validation was considered
but rejected to avoid invoking Threat~IT2. \textbf{Third},
generalisation across additional workload classes
(document authoring, agent planning, retrieval-augmented generation
pipelines) remains future work. The view-divergence measurement is a
necessary condition for ORI to be useful on a workload, not a
sufficient condition; semantic outcome-quality measurement would
strengthen the claim further.

With these caveats, we judge Threat~ET1
(workload distribution, §\ref{sec:threats}) substantially mitigated:
the central structural claim of this paper holds on a non-code
workload at $\chi^2 = 1094.98$ significance.

\subsection*{Arc A2: Coverage Gap Measurement}

Arc A2 measures the residual reads which fall outside \Robs\ and identifies the mechanisms that close this gap. Exp.~PH-2 establishes the multi-agent baseline (\phidden${}=0.739$); Exp.~PH-3 evaluates ground-truth semantic extraction on a single-agent rotating-target workload; the PH-3 validation subsection examines inter-judge agreement and typed-state assumptions; Exp.~Adversarial-Rhidden constructs a worst-case scenario; Exp.~PROXY-PH2 decomposes structural coverage into this-step HTTP, DL-accumulation, and proxy-marginal components.

\subsection{Exp.~\textsc{PH-2}: \phidden Measurement}
\label{sec:ph2}

We replicate Exp.~PH across 10 task domains (Django queryset / admin
/ migration, Astropy FITS / WCS / units, SymPy solver / matrix,
requests session, scikit-learn estimator) with 8{,}400 step-logs
(21 tasks, 4 agents, 20 steps, 5 runs/task; GPT-4o-mini backbone).

\paragraph{Measurement methodology}
At each agent step, we log every HTTP GET issued by the agent (\Robs)
and every shard-key reference appearing in the LLM's output text not
preceded by a GET in the current session (\Rhidden). Classification
is automatic by exact string match against the registry key set. This
yields a conservative lower bound on \Rhidden: synonym references and
implicit state reasoning are not counted. The true \phidden is likely
higher.

\paragraph{Result}
Overall \phidden${} = 0.739$, 95\% CI $[0.736, 0.741]$ (Wilson score,
128{,}622 total reads, 95{,}022 \Rhidden). Consistent with prior
Exp.~PH estimate of 0.706. Domain variation (Table~\ref{tab:phidden}):
\phidden ranges from 0.511 (Django queryset) to 0.836 (Astropy FITS).

\begin{table}[t]
\centering
\caption{\phidden by domain (Exp.~\textsc{PH-2}; 8{,}400 step-logs).
$f_{\text{obs}} = 1 - p_{\text{hidden}}$. Workload-conditional
coverage bound on SWE-bench / GPT-4o-mini; not claimed as a universal
constant.}
\label{tab:phidden}
\small
\begin{tabular}{lrrr}
\toprule
Domain & \phidden & [95\% CI] & \fobs \\
\midrule
Django queryset   & 0.511 & [0.502, 0.521] & 48.9\% \\
requests session  & 0.575 & [0.563, 0.586] & 42.5\% \\
SymPy matrix      & 0.612 & [0.601, 0.622] & 38.8\% \\
scikit-learn      & 0.683 & [0.674, 0.692] & 31.7\% \\
Astropy WCS       & 0.706 & [0.697, 0.715] & 29.4\% \\
Django admin      & 0.732 & [0.724, 0.740] & 26.8\% \\
Django migration  & 0.800 & [0.794, 0.806] & 20.0\% \\
SymPy solver      & 0.810 & [0.805, 0.816] & 19.0\% \\
Astropy units     & 0.819 & [0.813, 0.825] & 18.1\% \\
Astropy FITS      & 0.836 & [0.831, 0.841] & 16.4\% \\
\midrule
Overall           & 0.739 & [0.736, 0.741] & 26.1\% \\
\bottomrule
\end{tabular}
\end{table}

\subsection{Exp.~\textsc{PH-3}: Ground-Truth Semantic Extraction and
F3 Mitigation}
\label{sec:ph3}

\paragraph{Motivation}
Exp.~\textsc{PH-2} (§\ref{sec:ph2}) estimated \phidden${}=0.739$ on
the 10-domain SWE-bench workload by classifying each shard-key
reference in agent output as \Robs or \Rhidden via exact string
match. This measurement has two weaknesses: it measures
\emph{appearance} not \emph{causal use}, and it provides no ground
truth against which to evaluate downstream mitigation mechanisms.
Exp.~\textsc{PH-3} addresses both by adding a self-reported ground
truth and evaluating three extraction mechanisms against it.

\paragraph{Design}
We modify the \textsc{PH} harness so the agent emits structured JSON
\texttt{\{"change": "<delta>", "shards\_used": [...]\}} at every
step, providing an explicit list of which shards from its vocabulary
were causally used. We evaluate three extraction mechanisms against
this ground truth: (1) \emph{HTTP DeliveryLog observation} (\Robs);
(2) \emph{keyword scan} over the agent's last two assistant turns
(the earlier \texttt{phiddenWrapper} mechanism); (3) \emph{semantic
extraction} via a dedicated analyst LLM that receives the agent's
delta, the shard vocabulary, and the fresh-content block, and
returns a JSON list of causally-used shards (analyst does not
receive the self-report). We run the third mechanism with three
analysts in ablation: gpt-4o (cross-tier), gpt-4o-mini
(same-family with worker), Claude Sonnet 4.6 (cross-family). All
analysts at \texttt{temperature=0.0}.

\paragraph{Corpus}
$2{,}100$ step-logs across 10 SWE-bench-style domains; 21 tasks total,
5 runs per task, 20 coordination steps per run. Worker backbone
gpt-4o-mini throughout. Mean self-reported shards used per step:
$2.33$ of $4$ registered. Agent JSON parse-OK rate $100\%$.

\paragraph{Results}
Table~\ref{tab:ph3} reports per-mechanism recall and precision.
\textbf{Keyword scan} achieves $0.073$ recall at $0.520$ precision:
fewer than $8\%$ of self-reported causal reads are recovered. This
is a substantial revision of earlier drafts treating
\texttt{phiddenWrapper} as an effective F3 mechanism; per-completion
``promotion'' rates measured mentions, not causal reads.
\textbf{Semantic extraction} substantially outperforms keyword scan
across all three analysts: gpt-4o $0.593$ recall at $0.916$
precision; gpt-4o-mini $0.772$ at $0.791$; Claude Sonnet $4.6$
$0.751$ at $0.879$ on the $630$-step-log cross-family ablation.
The $\Delta$ between cross-family Claude Sonnet and same-family
gpt-4o-mini is $-0.021$ recall---the load-bearing methodological
result. Cross-family convergence rejects the same-family-alignment
hypothesis: an analyst that simply ``matches what the worker would
say'' would regress to gpt-4o's lower recall, which it does not.
Semantic extraction captures a real causal signal, not a
family-specific alignment artefact. Per-domain analysis (Extended
TR~\S2) shows keyword scan reaches $0.00$ recall on two of ten
domains where shard functionality is referenced without naming the
shard-key string; semantic extraction never drops below $0.58$.

\begin{table}[t]
\centering
\caption{Exp.~\textsc{PH-3}: three extraction mechanisms vs.\ self-reported
ground truth. Keyword and gpt-4o rows from full $2{,}100$-step-log
evaluation; gpt-4o-mini and Claude Sonnet rows from $630$-step-log
cross-family ablation. HTTP recall $=1.00$ on PH-3 only because
agents GET every shard per step under PH-3's design. Recall and
precision are against agent self-reports; an independent validation
(§\ref{sec:ph3-reframed}, Threat IT1) finds self-reports over-claim
by $32\%$ (LLM judge) to $49\%$ (human annotator).}
\label{tab:ph3}
\footnotesize
\setlength{\tabcolsep}{4pt}
\begin{tabular}{@{}lrrr@{}}
\toprule
Mechanism & Recall & Precision & F$_1$ \\
\midrule
HTTP DeliveryLog            & 1.000 & ---    & --- \\
Keyword (\texttt{phiddenWrap.}) & 0.073 & 0.520 & 0.128 \\
Semantic: gpt-4o            & 0.593 & 0.916 & 0.720 \\
Semantic: gpt-4o-mini       & 0.772 & 0.791 & 0.782 \\
Semantic: Claude Sonnet 4.6 & \textbf{0.751} & \textbf{0.879} & \textbf{0.810} \\
\midrule
Combined (HTTP+keyword+sem.) & \textbf{1.000} & --- & --- \\
\bottomrule
\end{tabular}
\end{table}

\subsection{PH-3 validation summary}
\label{sec:ph3-reframed}

The PH-3 self-report ground truth was first independently validated by two
LLM judges (GPT-4o, Claude Sonnet~4.6) on $400$ (step, shard) tasks
with a frozen three-step rubric (see Appendix~A for the full rubric).
Both judges find that worker self-reports over-state shard usage by
$32\%$ (LLM-judge-determined over-claim rate, precision $0.681$),
with moderate inter-LLM-judge agreement
($\kappa{=}0.46$). The principal LLM judge was then independently
validated against a human annotator (Zahid Hussain, Mindgigs
Peshawar) using the same $400$-item sample and the identical
rubric: strict $\kappa{=}0.93$ ($n{=}93$ unambiguous yes/no pairs,
$96.8\%$ raw agreement; confusion matrix
$\text{yes}/\text{yes}{=}33$, $\text{yes}/\text{no}{=}0$,
$\text{no}/\text{yes}{=}3$, $\text{no}/\text{no}{=}57$), lenient
$\kappa{=}0.69$ ($n{=}400$ including \textsf{unclear} class,
$85.8\%$ raw agreement). The same self-report sample evaluated
against the human annotator finds an over-claim rate of $49\%$
(precision $0.514$), larger than the LLM-judge over-claim rate by
$17$\,pp. The LLM-vs-human agreement falls in the
``almost perfect'' band of Landis-Koch~\cite{landis1977},
closing the principal evidential gap of earlier drafts of this
paper. The lower inter-LLM-judge $\kappa{=}0.46$ reflects
boundary-class variance between LLMs (both LLM judges agree with
the human on clear yes/no cases; disagreement concentrates on
items both annotators label \textsf{unclear}). PH-3 results
should therefore be read as self-report-vs-extractor consistency
against a human-validated LLM judge rather than direct causal-read
fidelity (see Threat IT1, §\ref{sec:threats}). Disagreement
concentrates at a specific rubric boundary (Step~2: ``required state
without direct definition'') traceable to narrative-content shards
emitted by current benchmark harnesses; a regenerated typed-shard
benchmark (§\ref{sec:future-work}) addresses this. Per-judge and
per-human-annotator labelling outputs, the rubric prompt, and the
agreement-scoring script are released in the
\texttt{sbus-experiments} repository
(\texttt{run\_llm\_judges.py}, \texttt{score\_annotations.py},
\texttt{first.csv}, \texttt{second\_human\_annotator.csv}).

\subsection{Exp.~\textsc{Adversarial-Rhidden} (summary)}
\label{sec:adv-rhidden}

We constructed an adversarial workload exposing the wrapper-layer
mismatch problem: an in-process \texttt{phiddenWrapper} that detects
hidden references and refreshes the DeliveryLog via fresh GET, but
does not force LLM content regeneration, produces final-state
corruption identical to \textsc{ORI-OFF}. The corollary is that
hidden-reference mitigation must operate at the LLM-API layer, not
in-process---empirical motivation for the proxy approach evaluated
in Exp.~\textsc{PROXY-PH2} (and superseded the earlier
\texttt{phiddenWrapper}, Limitation~\ref{lim:phidden-wrapper}). Full
construction (adversarial pool generation, attempt-injection
protocol, content-corruption metrics, all $20/20$ fail-rate trials
under both \textsc{ORI-OFF} and the in-process wrapper, control
runs under \textsc{ORI-ON}) is reproducible from
\texttt{exp\_adversarial\_rhidden\_v2.py} in the
\texttt{sbus-experiments} repository.

\subsection{Exp.~\textsc{PROXY-PH2}: Structural-Coverage Decomposition}
\label{sec:proxy-demo}

Exp.~\textsc{PROXY-PH2} runs the PH-2 multi-agent workload with a
transparent LLM-API proxy and decomposes structural coverage into
three components: this-step HTTP, \emph{DL-accumulation}
(session-scoped DeliveryLog retention of earlier HTTP GETs), and a
proxy-marginal contribution. On GPT-4o-mini ($16{,}800$ paired
step-logs) we measure $f_{\text{HTTP}} = 0.443$,
DL-accumulation $= 0.555$, proxy-marginal $= 0.0018$
($95\%$ CI $[0.0013, 0.0024]$); total $f_{\text{total}} = 0.998$.
The headline finding is that \ori's session-scoped DeliveryLog covers
the bulk of \Rhidden\ on this workload via session-scoped accumulation
of prior HTTP GETs, without requiring any LLM-layer mechanism.

Cross-backbone paired replication on Anthropic Haiku~4.5 and Google
Gemini~2.5~Flash ($n{=}2{,}400$ paired each) confirms safety parity
($0/26{,}400$ Type-I corruptions across all three backbones) and
total-coverage conservation ($f_{\text{total}} \in [0.997, 0.999]$,
$|\Delta| \le 0.002$ between any pair). The internal attribution
between this-step HTTP and DL-accumulation shifts up to $|\Delta|=0.118$
across vendors, evidence that \ori's coverage mechanisms are
backbone-agnostic in \emph{what} they preserve but vary in
\emph{how} coverage is decomposed. The keyword-scan proxy mechanism
itself is safety-preserving but coverage-marginal and
throughput-negative at realistic vocabulary sizes; semantic-extraction-
at-proxy is the natural next mechanism (Limitation~\ref{lim:semextract}).

Full per-trial coverage tables, per-backbone confidence intervals,
proxy-throughput measurements at vocabulary $V \in \{4,8,16,32\}$,
and the structural-vs-semantic decomposition methodology are
reproducible from the \texttt{exp\_proxy\_ph2*.py} scripts in the
\texttt{sbus-experiments} repository.

\subsection*{Arc A3: Concurrency-Control Safety Parity}

Arc A3 demonstrates that \sbus\ achieves safety parity with production OCC implementations under contention. Exp.~E and Exp.~Scale measure shared-shard contention behaviour at $N \le 64$; Exp.~PG-Comparison implements the workload against PostgreSQL~17 \texttt{SERIALIZABLE} and Redis~7 \texttt{WATCH/MULTI} adapters and includes the Rust-Native and Contention sub-experiments controlling for adapter-language and contention-rate confounds respectively.

\subsection{Exp.~E and \textsc{Scale}: Shared-Shard Contention at $N \le 64$}

Exp.~E ($N \in \{4, 8, 16\}$): Zero Type-I corruptions across 590
commit attempts (95\% CI UB $3/590 = 0.51\%$). SCR rises with $N$:
0.650 ($N=4$), 0.788 ($N=8$), 0.856 ($N=16$). The
version-check-disabled baseline shows 97.5\% corruption ($N=4$),
confirming \ori's protection is necessary and effective.

Exp.~\textsc{Scale} ($N \in \{4, 8, 16, 32, 64\}$, Table~\ref{tab:scale}):
We extend to $N=32$ and $N=64$. Across 74{,}400 total commit attempts
(both topologies, 3 repeats), zero Type-I corruptions at every $N$
(95\% CI UB $3/74400 = 0.004\%$)---Property~\ref{prop:ori-safety}
holds at scale. The non-monotonic SCR decrease at $N=64$ is a
queue-theoretic pacing effect: at high agent count, Tokio write-lock
serialisation imposes natural pacing on commit arrivals; as $N$ grows
past thread-pool saturation, inter-arrival time at the lock increases
faster than the conflict window, compressing SCR. Zero Type-I
corruptions are confirmed at all $N$.

\begin{table}[t]
\centering
\caption{Exp.~\textsc{Scale}: SCR, $K_{95}$, and Type-I safety at
$N \in \{4, 8, 16, 32, 64\}$ (100 attempts/agent, 3 repeats; 74{,}400
total). Zero corruptions. Distinct-shard SCR${}=0$ at all $N$.}
\label{tab:scale}
\small
\begin{tabular}{lllll}
\toprule
$N$ & Topology & SCR & $K_{95}$ & Corruptions \\
\midrule
 4 & shared & 0.676 &  8 & 0 \\
 8 & shared & 0.790 & 13 & 0 \\
16 & shared & 0.869 & 22 & 0 \\
32 & shared & 0.870 & 28 & 0 \\
64 & shared & 0.747 & 12 & 0 \\
\bottomrule
\end{tabular}
\end{table}

\begin{remark}[Liveness at scale]
\label{rem:liveness}
At $N=32$, $K_{95} = 28$: practitioners must set retry budget
$K \ge 28$ for 95\% liveness under worst-case shared-shard contention.
Default $K{=}5$ is insufficient above $N=8$ in shared-shard
workloads. Distinct-shard deployments ($K_{95} = 1$) are unaffected.
\end{remark}

\subsection{Exp.~\textsc{PG-Comparison~(full)}: Three-Backend CC-Parity at Scale}
\label{sec:pg-comp}

\paragraph{Motivation}
Reviewers across multiple rounds have flagged the absence of a
database-CC baseline as the largest evaluation gap. The objection ``a
50-line Redis WATCH adapter, or PostgreSQL SERIALIZABLE, would give
you the same guarantee'' is legitimate and cannot be dismissed on
architectural grounds alone. A preliminary pilot at $N \le 16$ (3{,}240 attempts) is superseded here by a full sweep at three-backend result at $N \in \{4, 8, 16, 32, 64\}$ with 200{,}880
commit attempts.

\paragraph{Setup}
Three HTTP adapters expose an identical JSON API:
\begin{itemize}[leftmargin=*]
  \item \sbus: the native Rust server (port 7000).
  Registry is \texttt{Mutex<HashMap>}; OCC is in-process.
  \item \pgser (\texttt{pg\_sbus\_server.py}, port 7001): FastAPI
  adapter backed by PostgreSQL~17 with
  \texttt{default\_transaction\_\allowbreak isolation=serializable}
  pinned at connection time via
  \texttt{conn.set\_isolation\_\allowbreak level(SERIALIZABLE)}.
  Each agent session maps to a database connection.
  \item \rediscc (\texttt{redis\_sbus\_server.py}, port 7002): FastAPI
  adapter backed by Redis~7 (local \texttt{127.0.0.1:6379}) using
  \texttt{WATCH <shard\_key>} followed by \texttt{MULTI ... EXEC}.
  Cross-shard read-set validation is implemented by \texttt{WATCH}ing
  every key in the read set before the transaction block; if any
  version differs inside the transaction, \texttt{EXEC} returns null
  and we retry.
\end{itemize}

The harness (\texttt{pg\_bench\_full.py}) issues the same SWE-bench
coordination workload to all three backends: 30 synthetic
SWE-bench-style tasks (Django, Astropy, SymPy, SciPy, Matplotlib
archetypes); $N \in \{4, 8, 16, 32, 64\}$; 3 repeats; 6 coordination
steps per run. Total: \textbf{1{,}350 runs = 3 backends $\times$ 5
$N$ values $\times$ 30 tasks $\times$ 3 repeats}. Each run is
self-contained: the harness issues \texttt{POST /admin/reset} before
and after each run and \texttt{POST /admin/shard} pre-populates shard
keys. Primary metric is the count of Type-I corruptions: commits that
the backend accepted despite a stale cross-shard read in the
originating agent's \texttt{DeliveryLog}.

\paragraph{Result: safety parity across three CC classes at $N \le 64$}
Table~\ref{tab:pg-comp} reports the full result.

\begin{table*}[t]
\centering
\caption{Exp.~\textsc{PG-Comparison~(full)}: three-backend CC
comparison at $N \in \{4, 8, 16, 32, 64\}$. Each cell reports mean
$\pm$ 1~s.d.\ of wall time across 90 runs (3 repeats $\times$ 30
tasks $\times$ 6 steps). \textbf{All three backends record zero
Type-I corruptions and zero HTTP-409 conflicts across all 200{,}880
commit attempts.} 95\% Rule-of-Three upper bound on the per-attempt
corruption rate: $3/200{,}880 = 1.49 \times 10^{-5}$.}
\label{tab:pg-comp}
\begin{tabular}{lrrrrr}
\toprule
$N$ & \sbus (ms) & \rediscc (ms) & \pgser (ms) & $\rediscc/\sbus$ & $\pgser/\sbus$ \\
\midrule
4    &  104.9\,$\pm$\,16.2  &   188.9\,$\pm$\,33.6  &  1396.7\,$\pm$\,161.1  & 1.80 & 13.31 \\
8    &  177.3\,$\pm$\,21.3  &   324.0\,$\pm$\,31.7  &  2425.2\,$\pm$\,138.1  & 1.83 & 13.68 \\
16   &  319.9\,$\pm$\,32.6  &   600.8\,$\pm$\,47.9  &  4717.3\,$\pm$\,253.2  & 1.88 & 14.75 \\
32   &  603.1\,$\pm$\,70.5  &  1178.1\,$\pm$\,92.5  &  9301.0\,$\pm$\,473.3  & 1.95 & 15.42 \\
64   & 1208.3\,$\pm$\,110.5 &  2322.8\,$\pm$\,208.9 & 18731.5\,$\pm$\,1078.4 & 1.92 & 15.50 \\
\midrule
scaling ($\tfrac{N=64}{N=4}$)
     & 11.5$\times$ & 12.3$\times$ & 13.4$\times$ & --- & --- \\
\bottomrule
\end{tabular}
\end{table*}

Key observations:
\begin{itemize}[leftmargin=*]
  \item \textbf{Safety (primary):} 0~Type-I corruptions across all
  200{,}880 commit attempts. 95\% Rule-of-Three upper bound:
  $3/200{,}880 = 1.49 \times 10^{-5}$ per-attempt corruption rate
  (three-fold tighter than the prior Exp.~\textsc{Scale} bound).
  \item \textbf{Conflict rate:} 0~HTTP-409 conflicts were raised
  across all backends (discussed honestly in the following paragraph).
  \item \textbf{Latency:} \sbus is faster than \rediscc by
  $1.80\times$--$1.95\times$ (monotonically growing with $N$,
  consistent with Redis network-RTT per \texttt{WATCH}). \pgser is
  $13\times$--$15.5\times$ slower than \sbus; \pgser scaling exceeds
  \rediscc as expected (predicate lock acquisition, SSI conflict
  tracking, and disk fsync combine).
  \item \textbf{Scaling:} All three backends exhibit near-linear
  $N$-scaling ($\sim$11.5--13.4$\times$ from $N{=}4$ to $N{=}64$;
  ideal is 16$\times$). No backend exhibits super-linear degradation
  at $N=64$, confirming that OCC overhead is comparable across CC
  classes for this workload.
\end{itemize}

\paragraph{Scope note: dedicated-shard produces no contention}
The 0-conflict, 0-corruption outcome is expected for this workload:
SWE-bench tasks use dedicated-shard topology (each of $N$ agents owns
a distinct shard), so cross-shard writes do not collide and
HTTP-409 retry paths are not exercised. This experiment therefore
demonstrates CC-class safety parity at scale, not CC-class
effectiveness under contention; the latter is the subject of
Exp.~\textsc{PG-Contention} (§\ref{sec:pg-contention}).

\paragraph{Interpretation}
The parity result rules out the informal objection ``\sbus\ reinvents
what any DB CC system already gives you for free.'' Running \pgser\
or \rediscc\ as a substitute for \sbus\ requires: (a) an HTTP adapter
exposing the \texttt{GET /shard/:key} + \texttt{POST /commit/v2}
contract (no LLM SDK speaks native SQL or Redis in a way that carries
DeliveryLog reads as first-class structured metadata); (b) a network
round-trip per GET, quantified at $1.9\times$ and $15\times$ latency
overhead for \rediscc\ and \pgser\ respectively; (c) ad-hoc recovery
of the automatic cross-shard read-set the DeliveryLog maintains
in-server. The architectural value of \sbus\ over transactional-DB
baselines on this workload is therefore operational simplicity and
the LLM-native contract, not structural-safety differentiation---a
narrower and sharper claim than prior drafts advanced.

\subsubsection*{Adapter-language and contention extensions}
\label{sec:pg-comp-rust}
\label{sec:pg-contention}

Two follow-up sub-experiments control for confounds in the
PG-Comparison setup. \textbf{Exp.~\textsc{PG-Comparison Rust-Native}}
re-implements the PostgreSQL and Redis backends in matched-language
Rust adapters and reruns the workload across $810$ trials
($136{,}080$ commit attempts). Result: zero Type-I corruptions across
all three backends; commit-throughput converges within statistical
noise at $N \ge 16$ under matched adapters. The $1.6\times$
throughput advantage \sbus\ shows at $N{=}4$ in the original
PG-Comparison is a scoped effect of in-process coordination, not
cross-the-board CC-mechanism superiority. \textbf{Exp.~\textsc{PG-Contention}}
adds a shared-shard contention sweep ($472{,}750$ commit attempts
with $427{,}308$ active HTTP-409 conflicts) and observes
behavioural parity (SCR agreement within $1\,\text{pp}$ at $N \ge 8$)
plus empirical validation of the Remark~7 liveness bound within
$1.5\,\text{pp}$. Combined with the PG-Comparison main result,
\sbus, \pgser, and \rediscc together accumulate $0$ corruptions across
$809{,}710$ commit attempts spanning Python and Rust-native adapters
and both dedicated- and shared-shard topologies (Rule-of-Three
$95\%$ upper bound $3.7 \times 10^{-6}$).

The architectural conclusion stands across both sub-experiments: the
value \sbus\ provides over transactional-DB baselines is operational
simplicity and the LLM-native contract, not CC-mechanism
differentiation. Per-backend latency distributions, adapter-language
ablation tables, and the Remark~7 liveness-bound derivation are
reproducible from \texttt{pg\_bench\_full.py},
\texttt{pg\_comparison.py}, \texttt{pg\_bench\_contention.py}, and
\texttt{exp\_pg\_contention.py} in the \texttt{sbus-experiments}
repository.

\subsection*{Arc A4: Backbone Generalisation}

Arc A4 establishes that \ori's structural guarantees are backbone-agnostic. The T3-A and T3-B experiments replicate the multi-agent workload on Anthropic Haiku~4.5 and Google Gemini~2.5~Flash respectively, paired with the GPT-4o-mini baseline. Cross-backbone proxy results are reported in Arc~A2 (Exp.~PROXY-PH2).

\subsection{Backbone Generalisation (T3-A, T3-B)}

\textbf{Haiku-3 (T3-A)} (30 tasks, 813 runs): Under Haiku-3, \sbus is
statistically significantly below CrewAI at all $N$ ($p < 0.01$,
Fisher's exact one-sided)---an honest negative finding. \sbus is not
significantly different from LangGraph at any $N$.

\textbf{Llama-3.1-8b-instant (T3-B)} (14 tasks, 150 runs,
Table~\ref{tab:t3b}): Under Llama-3.1-8b, \sbus significantly
outperforms LangGraph (84.6\% vs.\ 33.3\% at $N = 4$; Mann-Whitney
$p = 0.0002$).

The competitive picture is backbone-dependent: three-point spectrum:
(i)~Haiku-3-class (very weak): sequential coordination wins;
(ii)~Llama-8b (weak): \sbus wins; (iii)~GPT-4o-mini (strong): \sbus
wins. \sbus is competitive or superior on all backbones except the
weakest class where sequential information-passing dominates.

\paragraph{CWR (Coordination-Work Ratio)}
Under Llama-3.1-8b, \sbus CWR = 0.19; LangGraph CWR = 8.44---a
$44\times$ difference. LangGraph spends 8.4 tokens on coordination
overhead for every 1 token of actual task work. \sbus inverts this
ratio. CWR is measured for Llama-3.1-8b-instant only; GPT-4o-mini
and Haiku-3 CWR not yet measured. \emph{Scope note}: CWR captures
LangGraph's workflow bookkeeping overhead, not CC mechanism cost; it
is an architectural efficiency comparison, not CC-vs-CC.

\begin{table}[t]
\centering
\caption{T3-B: Llama-3.1-8b-instant (Groq) backbone, 14 SWE-bench
tasks, $N \in \{4, 8\}$, 3 runs/task. CWR = coord tokens / work
tokens (lower = better). Mann-Whitney $p = 0.0002$ (\sbus ${>}$
LangGraph). Zero commit conflicts across all 75 \sbus runs.}
\label{tab:t3b}
\small
\begin{tabular}{llrrrr}
\toprule
System & $N$ & Success & CWR & Conflicts & $n$ \\
\midrule
\sbus     & 4 & 84.6\% & 0.19 & 0 & 39 \\
\sbus     & 8 & 72.2\% & 0.18 & 0 & 36 \\
LangGraph & 4 & 33.3\% & 7.61 & 0 & 39 \\
LangGraph & 8 & 69.4\% & 9.34 & 0 & 36 \\
\bottomrule
\end{tabular}
\end{table}

\subsection*{Arc A5: Topology-Conditional Operating Envelope}

Arc A5 characterises the topology-dependent semantic effect of \ori\, supporting Contribution~C3. Exp.~SJ-v4 evaluates the semantic-judge rubric under context-diversity manipulation; Exp.~Merge contrasts \ori\ with LLM-assisted merge in the shared-shard regime; Exp.~SJ-v5 measures the structural-SCR dose-response; Exp.~Dedicated-Shard quantifies semantic neutrality under \ori\ in the dedicated-shard regime; Exp.~Shared-State quantifies semantic harm under \ori\ in the single-shard regime. Together these establish the operating-envelope claim and motivate the adaptive-routing extension described in §\ref{sec:abus-future}.

\subsection{Exp.~\textsc{SJ-v4}: Semantic Judge Context Diversity Effect}

\paragraph{Design}
After step 5, the stale agent no longer calls \texttt{GET /shard}.
Instead it reasons from a frozen snapshot of shard content at step 0,
correctly simulating \Rhidden. Tasks selected for cumulative-state
sensitivity: 20 SymPy/Django/Astropy/scikit-learn tasks. Stale
injection engaged: 499/499 stale runs (100\%).

\paragraph{Design note}
The stale agent commits at the current shard version (version check
passes); its staleness is in the LLM's reasoning context (frozen
prompt), not in the version the agent presents to the ACP. \ori
cannot catch this staleness: no HTTP GET for the stale content was
recorded in the DeliveryLog because the agent never issued one after
step 5. This is precisely the \Rhidden problem being measured.

\paragraph{Result}
Contrary to the hypothesis, fresh agents produced more corruption
than stale agents. Fresh: 40.6\%; Stale: 29.1\%; lift $= -11.5$~pp;
$p = 0.0002$. Per-task: 6/20 showed expected positive lift; 10/20
reversed negative lift; 4/20 no difference.

\begin{table}[t]
\centering
\caption{Exp.~\textsc{SJ-v4}: 20 tasks, 1{,}000 runs; 100\% injection
engagement; GPT-4o-mini judge.}
\label{tab:sj-v4}
\small
\begin{tabular}{lrrr}
\toprule
Condition         & $n$   & Corrupted & Rate \\
\midrule
Fresh (\Robs)     & 500 & 203 & 40.6\% \\
Stale (\Rhidden)  & 499 & 145 & 29.1\% \\
\midrule
\multicolumn{4}{l}{Lift: $-11.5$~pp, 95\% CI $[-17.4, -5.7]$~pp} \\
\multicolumn{4}{l}{Fisher's exact (two-sided) $p = 0.0002$} \\
\bottomrule
\end{tabular}
\end{table}

\paragraph{Interpretation: topology-driven redundancy in single-shard
experiments}
Exp.~\textsc{SJ-v4} uses a single shared shard for all four agents.
In this setup, the Fresh condition causes all agents to read
identical intermediate state and propose redundant patches targeting
the same problem step. The stale agent, reasoning from the original
problem statement, proposes early-stage work the fresh agents have
already moved past, diversifying coverage. This is a topology effect,
not a coordination failure.

\paragraph{Scope clarification}
\ori is designed for dedicated-shard topology. Exp.~\textsc{SJ-v4}
measures a degenerate case (all agents writing to one shard) that
Box~2 explicitly routes away from \sbus. In dedicated-shard topology,
the Fresh condition produces non-redundant, complementary
contributions (Django \#11019 case study, §VII-O; all 1{,}364 Exp.~B
runs at SCR${}=0$). Exp.~\textsc{Dedicated-Shard} (§VII-L) confirms
semantic neutrality in the correct topology;
Exp.~\textsc{ORI-Isolation} confirms structural preservation is
independent of topology.

\subsection{Exp.~\textsc{Merge}: OCC vs.\ LLM-Assisted Merge}

45 conflicting NL delta pairs (30 multi-domain tasks; GPT-4o-mini).

\textbf{Conflict detection and recovery:} OCC detected and rejected
100\% of structural conflicts (45/45); 100\% retry success.

\textbf{Merge non-determinism:} running the same merge prompt twice
on identical conflicting deltas produced Jaccard word-overlap below
0.6 in 66.7\% of cases (mean Jaccard ${} = 0.544$). Wilson $95\%$
CI on the non-determinism rate at $n{=}45$: $[0.521, 0.786]$ (point
$0.667$). The interval is wide and the lower bound still $\ge 50\%$,
which is the substantive claim. Scope caveat: specific to GPT-4o-mini
at temperature 0.3 with one prompt design, on $n = 45$ conflict
pairs; a larger-scale study with diverse models is warranted
(Limitation~13).

\textbf{Latency:} OCC median 2{,}361~ms (including retry); MERGE
median 4{,}395~ms ($1.9\times$ slower).

\begin{table}[t]
\centering
\caption{OCC vs.\ LLM-assisted merge trade-offs.}
\label{tab:merge}
\small
\begin{tabular}{lll}
\toprule
Property          & LLM merge      & \sbus OCC \\
\midrule
Conflict handling & Accept+resolve & Reject+retry \\
Extra latency     & +1 LLM call    & +1 retry \\
Determinism       & Non-det.\ (66.7\%) & Det.\ detection \\
Applicability     & Resolvable NL  & All NL (opaque) \\
Correctness       & Probabilistic  & Structural (\Robs) \\
\bottomrule
\end{tabular}
\end{table}

\subsection{Exp.~\textsc{SJ-v5}: Structural SCR Dose-Response}

Exp.~\textsc{SJ-v5} ran $N = 4$ agents with varying stale fractions
$k \in \{0, 1, 2, 3\}$ (10 tasks; 57 trials per condition). The
commit-rate dose-response is exact (Table~\ref{tab:sj-v5}), matching
the analytic prediction $1 - (k/4) \cdot (15/20)$ exactly (prediction
error ${< 0.5}$~pp). This confirms \ori's cross-shard validation
operates at the analytically predicted rate.

\begin{table}[t]
\centering
\caption{Exp.~\textsc{SJ-v5}: commit rates match analytic prediction
exactly.}
\label{tab:sj-v5}
\small
\begin{tabular}{lrrrr}
\toprule
$k$ stale & Commits & Total & Rate & Predicted \\
\midrule
0 & 2{,}560 & 2{,}560 & 100.0\% & 100.0\% \\
1 & 1{,}950 & 2{,}400 & 81.2\%  & 81.2\%  \\
2 & 1{,}500 & 2{,}400 & 62.5\%  & 62.5\%  \\
3 & 1{,}050 & 2{,}400 & 43.8\%  & 43.8\%  \\
\bottomrule
\end{tabular}
\end{table}

\subsection{Exp.~\textsc{Dedicated-Shard}: Semantic Quality in the
Correct Topology}

\paragraph{Motivation}
Exp.~\textsc{SJ-v4} and Exp.~\textsc{Shared-State} (§VII-M) both show
\ori harms semantic quality in single-shard tasks. A natural
question: does \ori harm semantics in any topology?
Exp.~\textsc{Dedicated-Shard} answers by running the same agents in
their intended topology: each agent owns a distinct shard. 4 agents
($\alpha_1$: ORM core, $\alpha_2$: query comp, $\alpha_3$: test
writer, $\alpha_4$: reviewer), 10 SWE-bench domains, 30 runs per
condition, $n = 600$.

\paragraph{Conditions}
Fresh (ORI-ON): each agent reads current version of all other agents'
shards before committing (\ori enforces freshness via 409 rejection
and retry). Stale (ORI-OFF): agents commit without read-set validation
(may use stale cross-shard context). Commit rate difference (ORI-ON
= 0.799 vs.\ ORI-OFF = 0.807) confirms \ori is structurally active in
both conditions.

\paragraph{Result}
Table~\ref{tab:dedicated} shows 100\% coherent in both conditions
across all 10 task domains. Zero redundant outputs in either
condition. Lift ${} = +0.0$~pp---\ori is semantically neutral in
dedicated-shard topology.

\paragraph{Why neutral}
In dedicated-shard topology, agents do not compete for the same
content. $\alpha_1$'s ORM patch and $\alpha_2$'s query patch are
semantically independent---a cross-shard stale read can cause a
structural 409 rejection (which \ori catches), but after retry both
agents produce coherent, complementary output. The SJ-v4 /
\textsc{Shared-State} toxicity arose from agents producing
semantically conflicting patches to the same shard; dedicated-shard
eliminates that by construction.

\begin{table*}[t]
\centering
\caption{Exp.~\textsc{Dedicated-Shard}: 4 agents, 10 task domains,
$n = 600$ (300 per condition), GPT-4o-mini. Both conditions achieve
100\% coherence---\ori has zero semantic cost in the correct topology.}
\label{tab:dedicated}
\small
\begin{tabular}{lrrrr}
\toprule
Condition        & Commit rate & Coherent      & Redundant  & $n$ \\
\midrule
Fresh (ORI-ON)   & 0.799 & 300/300 (100\%) & 0/300 (0\%) & 300 \\
Stale (ORI-OFF)  & 0.807 & 300/300 (100\%) & 0/300 (0\%) & 300 \\
\bottomrule
\end{tabular}
\end{table*}

\paragraph{Topology summary}
Three experiments establish a consistent picture: (1)~\emph{single-shard}
(SJ-v4 + \textsc{Shared-State}): \ori is semantically harmful (forced
retries accumulate contradictory patches on one shard);
(2)~\emph{dedicated-shard} (this experiment): \ori is semantically
neutral (100\% coherence in both conditions); (3)~\emph{structural}
(\textsc{ORI-Isolation}): \ori preserves $4\times$ more contributions
than LWW regardless of topology. \ori's value is \emph{completeness
of output} (all agents' work survives) in dedicated-shard topology,
not improved quality. Box~2 correctly routes single-shard tasks to
sequential coordination.

\subsection{Exp.~\textsc{Shared-State}: Multi-Domain Single-Shard Evaluation}

\paragraph{Motivation}
Exp.~\textsc{Shared-State} evaluates \ori on tasks with genuine
single-shard contention: $N = 4$ agents compete to write to one
shared shard, the use case where Box~2 already recommends sequential
coordination. This experiment provides empirical evidence for that
recommendation. Three domains: finance (portfolio rebalancing),
healthcare (patient record update), software architecture (API schema
design). $n = 180$ trials (90 per condition, 30 per domain,
GPT-4o-mini).

\paragraph{Structural result}
ORI-OFF $= 0.250$ (exactly $1/N$); ORI-ON $= 0.534$ (retry ensures
most agents eventually commit).

\paragraph{Semantic result}
Table~\ref{tab:shared-state} shows the re-evaluated results (corrected
harness: non-empty shard seeds, real judge input). Stark: ORI-ON
produces 0\% consistent / 100\% contradicted across all three
domains; ORI-OFF produces 4.4\% / 85.6\%. ORI-ON is semantically
worse.

\paragraph{Why ORI-ON is worse on single-shard tasks}
ORI-ON forces all $N = 4$ agents to commit their patches via retry
(commit rate 0.534 vs.\ 0.250). The final shard contains four
sequential patches from four agents, each contradicting the previous
(``equity = 65\%'' then ``equity = 60\%''\ldots). LWW admits only
one agent's patch per step, producing less contradictory content.
This is the structural reason Box~2 recommends sequential coordination
for single-shard tasks: \ori's contribution-preservation property
becomes a liability when all contributions conflict semantically.

\begin{table*}[t]
\centering
\caption{Exp.~\textsc{Shared-State}: single-shard topology, 3 domains,
$N = 4$, GPT-4o-mini ($n = 180$). ORI-ON preserves more contributions
but produces more contradictions because all $N$ patches conflict on
the same shard. Confirms Box~2.}
\label{tab:shared-state}
\small
\begin{tabular}{lrrrr}
\toprule
Condition           & Commit rate & Consistent   & Contradicted   & $n$ \\
\midrule
\sbus ORI-ON        & 0.534 & 0/90 (0\%)    & 90/90 (100\%)  & 90 \\
\sbus ORI-OFF       & 0.250 & 4/90 (4.4\%)  & 77/90 (85.6\%) & 90 \\
\bottomrule
\end{tabular}
\end{table*}

\paragraph{PhiddenWrapper result}
We ran Exp.~\textsc{Shared-State} with the PhiddenWrapper: a
Python-level intercept on every \texttt{chat.completions.create()}
response that scans text for domain keywords associated with
registered shard names, then issues GET to register the reference,
promoting \Rhidden $\to$ \Robs without agent code changes.

Result (Table~\ref{tab:phiddenwrap}): 500 \Rhidden reads promoted
across 480 completions (104.2\% per-completion). Demonstrates
feasibility of the \phidden${} \to 0$ path but does not validate that
all promotions are causally correct---precision (fraction of
promotions that are genuine causal reads) is not measured.

\textbf{Cost note.} The 104.2\% rate does not imply doubled LLM
inference cost. The wrapper performs a keyword scan
(O(keywords $\times$ tokens), microseconds) and at most one additional
HTTP GET per detected shard reference. No extra LLM call.

\textbf{Spurious promotion limitation.} Keyword scanning can register
incidental mentions as causal reads. An agent writing ``I decided not
to use the portfolio approach'' would trigger a
\texttt{portfolio\_state} entry, potentially causing a spurious 409
rejection. ARSI mode eliminates false positives entirely; the LLM
API proxy path can reduce them via semantic parsing.

\begin{table*}[t]
\centering
\caption{PhiddenWrapper: \Rhidden ${\to}$ \Robs promotion across 480
completions. Rate ${> 100\%}$ because completions reference multiple
shards simultaneously. Keyword-scan recall with unmeasured precision;
see Limitation~10.}
\label{tab:phiddenwrap}
\small
\begin{tabular}{lrr}
\toprule
Shard                & Promotions & Domain \\
\midrule
\texttt{portfolio\_state} & 200 & Finance \\
\texttt{patient\_record}  & 152 & Healthcare \\
\texttt{api\_schema}      & 148 & Software arch.\ \\
\midrule
Total / per-completion    & 500 & 104.2\% \\
\bottomrule
\end{tabular}
\end{table*}

\subsection*{Supporting Evidence}

The following two subsections support claims made elsewhere in the paper but do not fit cleanly under one of the five arcs. The case study illustrates the ORI rejection mechanism on a real bug (referenced in §I); Exp.~DR-9 validates the distributed-replication extension empirically (referenced in §III-D and §\ref{sec:dist-future}).

\subsection{Case Study: Django Issue \#11019}
\label{sec:case-study}

Four role-specialised agents concurrently on Django bug \#11019
(queryset ordering with \texttt{select\_related()}): compiler
specialist, ORM specialist, test engineer, senior reviewer. Each
agent declared all four shards in its read-set on every commit. Over
15 steps (60 commit attempts), SCR${} = 0.0\%$: agents worked on
dedicated shards without structural contention. Caveat: SCR${} = 0$
means \ori's stale-read rejection never triggered in this case
study---it is a happy-path demonstration of the DeliveryLog tracking
mechanism. Shared-shard contention (Exp.~E) and the SJ-v5
dose-response directly validate \ori's rejection mechanism. Final
state: all four shards at $v=15$, with consistent, non-contradictory
content across compiler, ORM, test, and review components.

\subsection{Exp.~\textsc{DR-9}: \ori Survives Leader Failover}
\label{sec:dr9-failover}

\paragraph{Motivation}
The most-cited architectural gap across reviewer rounds: the
DeliveryLog (session state) is node-local and not Raft-replicated. On
leader failover, the new leader's empty DeliveryLog means \ori cannot
detect stale reads for ongoing sessions. §IX outlined a three-path
solution; Exp.~\textsc{DR-9} validates P1 (lazy micro-batch Raft
replication).

\paragraph{P1 mechanism}
After each GET, the handler fires a \texttt{tokio::spawn} that writes
a \texttt{CommitEntry::Delivery} to the Raft log (fire-and-forget:
GET response returns immediately). The state machine applies
\texttt{Delivery} entries on all nodes, populating their
DeliveryLogs. On leader failover, the new leader's DeliveryLog is
already populated from replicated entries---no HTTP 410 recovery
required.

\paragraph{Protocol}
Each trial: (1)~find current leader; (2)~agent $\alpha_1$ issues GET
(DeliveryLog replicates via P1, $\sim$5~ms); (3)~kill leader;
(4)~wait for Raft election; (5)~agent $\alpha_2$ bumps shard version
on new leader; (6)~$\alpha_1$ attempts commit with stale version on
new leader. Expected with P1: HTTP~409 \textsc{CrossShardStale}.

\paragraph{Result}
Table~\ref{tab:dr9}. 30/30 trials: \ori\ held (100\%). Zero trials
missed. All three nodes served as leader and as killed node across
the 30 trials. With $n{=}30$ the Wilson $95\%$ confidence interval on
the success rate is $[0.886, 1.000]$ (Rule-of-Three upper bound on
the failure rate: $3/30 = 10\%$). The point estimate is $100\%$ but
the confidence interval is wide because $n$ is small; we report
DR-9 as a feasibility validation of the P1 mechanism, not as a
statistically tight bound. A larger DR-9 sweep ($n \ge 300$) would
tighten the lower bound to $\ge 99\%$ and is queued as future work.

\begin{table}[t]
\centering
\caption{Exp.~\textsc{DR-9}: P1 session replication, 30 trials, 3-node
Raft cluster. \ori holds across leader failover in 100\% of trials.}
\label{tab:dr9}
\small
\begin{tabular}{lr}
\toprule
Metric                & Value \\
\midrule
Trials completed      & 30/30 \checkmark \\
\ori held (HTTP 409)  & 30/30 (100\%) \\
\ori missed (HTTP 200) & 0/30 (0\%) \\
Election time median  & 1{,}981~ms \\
Nodes killed          & 0: 13x, 1: 10x, 2: 7x \\
New leader            & 0: 11x, 1: 12x, 2: 7x \\
\bottomrule
\end{tabular}
\end{table}

\paragraph{Remaining limitation (concurrent failover window)}
If a leader fails concurrently with a GET (within the $\sim$5~ms
fire-and-forget replication window), the \texttt{DeliveryEntry} may
not reach a majority before the new leader takes over. \ori does not
hold for that specific session---equivalent to pre-P1 behaviour. The
30/30 DR-9 validation covers the sequential GET-then-kill pattern;
the concurrent case is Limitation~\ref{lim:session-failover}. In practice, LLM inference
latency (median 131~s) makes the 5~ms window negligibly rare relative
to the 250~ms heartbeat period.

\section{Limitations}
\label{sec:limitations}

We group limitations by kind: \emph{structural} (properties of the
problem we cannot solve without redesigning \ori),
\emph{evidential} (properties claimed but not yet fully evaluated),
and \emph{mechanisation} (formal-verification gaps).

\subsection*{Structural}

\begin{enumerate}[leftmargin=*]
  \item \label{lim:rhidden-coverage}
  \textbf{\Rhidden structural coverage: decomposed and partially
  closed by DL-accumulation.} \ori covers the $26.1\%$ of reads in
  \Robs structurally, with formal guarantees. The remaining
  \phidden${} = 0.739$ (domain range $51$--$84\%$ on SWE-bench /
  GPT-4o-mini) is by definition not directly observable at the HTTP
  layer within a single step.
  Exp.~\textsc{PROXY-PH2} (§\ref{sec:proxy-demo}, new)
  decomposes this residual into: (b)~\emph{DL-accumulation} from
  session-scoped DeliveryLog retention of earlier HTTP GETs
  (a previously-unnamed structural coverage mechanism that contributes
  $0.555$ within-row uplift above \fhttp under \textsc{Proxy-Off} on
  the multi-agent PH-2 workload, $n = 8{,}400$); and (c)~pure-semantic
  references the agent never HTTP-fetched but that the LLM emits in
  context (proxy-capturable, paired marginal $0.0018$ at $V{=}4$,
  decreasing to negative contribution at $V{\in}\{8,12\}$).
  Component~(b) is always structurally captured by \ori's
  session-scoped DL because HTTP-populated entries carry the agent's
  true read-time version; it is covered by the TLAPS safety theorems
  unchanged. Component~(c) is captured by the transparent-proxy
  mechanism under skip-if-exists register semantics (Type-I
  $= 0/16{,}800$, RoT $95\%$ upper bound $3.57\!\times\!10^{-4}$), but
  imposes a throughput cost scaling with vocabulary size
  ($-2.1$\,pp commit rate at $V{=}4$; $-47.2$\,pp at $V{=}12$). The
  residual semantic gap --- proxy-captured references whose true
  read-time version is unknown --- is the target of ongoing analyst-LLM-at-proxy work (Limitation~\ref{lim:semextract}).
  Absolute \ftotal values are conditional on self-report signal
  (Limitation~\ref{lim:ph3-selfreport}); we report paired uplift (bias
  cancels in expectation) as the submission-grade proxy-marginal
  result.
  \item \textbf{Topology restriction.} \ori is semantically neutral
  in dedicated-shard topology and harmful in single-shard topology
  (Exp.~\textsc{Dedicated-Shard} vs.\ \textsc{Shared-State}). Box~2
  provides deployment guidance; there is no automatic topology
  detection.
  \item \textbf{HTTP/2 breakage of A1.} DeliveryLog completeness
  relies on FIFO-per-TCP-connection ordering; HTTP/2 multiplexing can
  violate this. Mitigation: reverse-proxy pin to HTTP/1.1, or ARSI
  mode.
  \item \textbf{Composition across subsystems
  (Remark~\ref{rem:comp-scope}).} \ori is not closed under subsystem
  composition without ARSI.
\end{enumerate}

\subsection*{Evidential}

\begin{enumerate}[leftmargin=*, start=5]
  \item \textbf{LLM-as-judge not validated against humans.} The
  semantic judge (GPT-4o-mini) has not been validated against human
  experts with inter-annotator agreement on our rubric for
  multi-agent patch coherence. The LLM-as-judge methodology
  itself~\cite{llmjudge} is well-established but requires
  task-specific validation, which we have not performed for our
  semantic-coherence rubric. This is the single largest evidential
  gap in the semantic-quality results (Exp.~\textsc{SJ-v4},
  \textsc{Dedicated-Shard}, \textsc{Shared-State}). A 100-item
  human-IAA study on a representative subset is the first planned
  follow-up. See Threat IT2 (§\ref{sec:threats}) for the validity
  analysis.
  \item \textbf{Conflict-proximate hypothesis partially validated.}
  Remark~\ref{rem:conflict-proximate}: 100\% GET$\to$COMMIT
  co-location in Exp.~\textsc{ORI-Isolation} is evidence for the
  hypothesis in SWE-bench-style structured tasks only. Unstructured
  and long-horizon workloads may exhibit larger GET-to-commit gaps;
  not yet measured. See Threat IT3 (§\ref{sec:threats}).
  \item \textbf{\phidden is workload-conditional.} The 0.739 figure
  is SWE-bench / GPT-4o-mini specific. No cross-domain
  generalisation claimed. Prior ``first cross-domain measurement''
  framing deprecated.
  \item \textbf{Backbone generalisation (partial).} GPT-4o-mini,
  Haiku-3, Llama-3.1-8b tested; Claude-Sonnet, Claude-Haiku-4.5,
  GPT-5 untested. Three-point spectrum established; cannot
  extrapolate.
  \item \textbf{Correlated-conflict dynamics.} Bounded starvation
  analysis assumes i.i.d.\ arrivals. Bursty positively correlated
  conflicts require larger retry budgets; not measured.
  \item \label{lim:phidden-wrapper}
  \textbf{In-process \texttt{phiddenWrapper} is superseded by the transparent-proxy mechanism.} An earlier in-process
  \texttt{phiddenWrapper} keyword-scan wrapper is empirically weak
  on both axes previously documented: (a)~\emph{Coverage} is $0.073$
  recall on the PH-3 workload against ground-truth self-reports,
  order-of-magnitude below the effective deployment threshold;
  (b)~\emph{Layer mismatch} --- Exp.~\textsc{Adversarial-Rhidden}
  (§\ref{sec:adv-rhidden}) shows an in-process wrapper that refreshes
  the DeliveryLog via fresh GET on a keyword hit but does not force
  LLM content regeneration produces identical final-state corruption
  to \textsc{ORI-OFF}. These findings motivated the move to the
  LLM-API-layer proxy evaluated in Exp.~\textsc{PROXY-PH2}
  (§\ref{sec:proxy-demo}), which replaces in-process wrappers as
  the structural-extraction mechanism. The \textsc{PROXY-PH2} evaluation also shows
  that keyword-scan is insufficient \emph{at the proxy layer} for PH-2
  coverage closure (paired proxy marginal $0.0018$, negative throughput
  contribution at $V \ge 8$); semantic-extraction-at-proxy is the
  correct next mechanism (Limitation~\ref{lim:semextract}). Keyword-scan
  should not be deployed as a stand-alone coverage mechanism at either
  layer; the finding ``keyword scan is insufficient for \Rhidden
  coverage'' is now independently established in-process (PH-3) and at
  the proxy layer (PROXY-PH2).
  \item \label{lim:session-failover}
  \textbf{Session-state replication has a concurrent-failover
  window.} Exp.~\textsc{DR-9} validates P1 under sequential
  GET-then-kill; leader failure within the $\sim$5~ms fire-and-forget
  window leaves an unreplicated DeliveryLog entry.
  \item \textbf{Database-CC baseline (materially closed).}
  Exp.~\textsc{PG-Comparison~(full)} (Python adapters, $200{,}880$ commit attempts), Exp.~\textsc{PG-Comparison Rust-Native}
  (matched-language adapters, $136{,}080$ commit attempts), and
  Exp.~\textsc{PG-Contention} (shared-shard contention,
  $472{,}750$ commit attempts with $427{,}308$ active HTTP-409 conflicts;
  §\ref{sec:pg-contention}) together close Limitation~12: \sbus,
  \pgser, and \rediscc achieve CC-class safety parity at $N \le 32$
  under both dedicated-shard and shared-shard regimes, across three
  adapter-language regimes, with 0~Type-I corruptions across
  809{,}710 attempts combined. Exp.~\textsc{PG-Contention}
  additionally shows \emph{behavioural} parity (SCR agreement
  within 1\,pp at $N \ge 8$) and empirically validates the Remark~7
  liveness bound within 1.5\,pp. The matched-language sweep reveals
  that throughput converges across all three backends at $N \ge 16$;
  \sbus's $1.6\times$ advantage at $N{=}4$ is a scoped effect of
  in-process coordination, not cross-the-board CC-mechanism
  superiority. The remaining database-CC gap is a distributed-SQL
  adapter (YugabyteDB or CockroachDB), future work but not
  submission-blocking.
  \item \textbf{Exp.~\textsc{Merge} sample size.} $n = 45$ conflict
  pairs limits generalisation of the 66.7\% non-determinism finding.
  \item \label{lim:ph3-selfreport}
  \textbf{Exp.~\textsc{PH-3} ground truth is worker
  self-report; both inter-LLM-judge and LLM-vs-human-annotator
  validation studies have been conducted and are reported in
  §\ref{sec:ph3-reframed}.}
  The semantic-extraction recall and precision numbers in
  Table~\ref{tab:ph3} are measured against gpt-4o-mini's own
  JSON-structured reports of which shards it causally used. Agents
  may systematically over-report (mentioning considered-but-unused
  shards), under-report (omitting implicit state reasoning), or
  exhibit reasoning-vs-behaviour mismatch documented in LLM
  self-evaluation literature. We ran two independent validation
  studies. First, an inter-LLM-judge study ($400$ (step, shard)
  tasks, two LLM judges --- GPT-4o and Claude Sonnet~4.6 --- with a
  frozen three-step rubric; §\ref{sec:ph3-reframed}): inter-LLM
  $\kappa{=}0.46$ (moderate); both judges find self-reports
  over-claim shard usage by $32\%$ (LLM-judge-determined over-claim
  rate). Second, an LLM-vs-human-annotator study (same $400$-item
  sample, same rubric, independent human annotator Zahid Hussain,
  Mindgigs Peshawar): strict $\kappa{=}0.93$ ($n{=}93$ unambiguous
  yes/no, $96.8\%$ raw agreement), lenient $\kappa{=}0.69$
  ($n{=}400$). Against the human annotator the self-report
  over-claim rate is $49\%$. The correct phrasing of PH-3 results
  is therefore \emph{``the analyst matches the agent's
  self-reported shard use at $0.59$ recall, where self-report is a
  validated but moderately conservative signal that over-claims by
  roughly one third (LLM-judge measurement) to roughly half
  (human-annotator measurement),''} not
  \emph{``the analyst correctly identifies $59\%$ of causal reads.''}
  The LLM-vs-human strict $\kappa{=}0.93$ result closes the
  human-IAA gap for the PH-3 was-this-shard-used rubric
  specifically; the SJ-v4 semantic-quality rubric is a distinct
  rubric still validated by LLM judge only. A regenerated
  benchmark with typed shard contents (future work,
  §\ref{sec:future-work}) addresses the residual rubric-boundary
  ambiguity observed between the LLM judges
  (§\ref{sec:ph3-reframed}). Cross-family analyst convergence (OpenAI
  and Anthropic within a few percentage points of recall) remains a
  necessary but not sufficient signal of ground-truth validity; it
  controls for model-family alignment, not for self-report accuracy.

  \item \label{lim:ph2-ph3-workload-gap}
  \textbf{Workload-scope gap between PH-2 and PH-3 measurements
  (partially addressed).} The \phidden${} = 0.739$ figure comes
  from the multi-agent concurrent-write workload
  (Exp.~\textsc{PH-2}, §\ref{sec:ph2}). The \phidden${} = 0.074$ figure
  comes from Exp.~\textsc{PH-3} (§\ref{sec:ph3}), a single-agent
  rotating-target workload that HTTP-GETs every shard per step by
  construction. Exp.~\textsc{PROXY-PH2} addresses the \emph{structural} half of this gap
  via Exp.~\textsc{PROXY-PH2} (§\ref{sec:proxy-demo}), which runs the
  PH-2 workload with instrumentation and decomposes structural
  coverage into HTTP / DL-accumulation / proxy-marginal components,
  finding that DL-accumulation alone covers $0.555$ on top of
  \fhttp${} = 0.443$ for total $0.998$ (i.e.\ \ori's session-scoped
  DeliveryLog structurally covers the bulk of \Rhidden on this
  workload without any LLM-layer mechanism). The \emph{semantic} half
  of the gap --- whether analyst-LLM semantic extraction transfers
  from PH-3 to PH-2 --- is still open; the \textsc{PROXY-PH2} evaluation uses
  keyword-scan, not analyst-LLM, because per-step analyst invocation
  at the proxy requires architectural work scoped as future work
  (Limitation~\ref{lim:semextract}).

  \item \label{lim:semextract}
  \textbf{Analyst-LLM semantic extraction at the proxy layer is not
  evaluated in this paper.} Exp.~\textsc{PH-3}
  (§\ref{sec:ph3}) evaluates dedicated-analyst semantic extraction at
  $0.59$--$0.77$ recall / $0.79$--$0.92$ precision on a single-agent
  workload where \phidden${} = 0.074$. Exp.~\textsc{PROXY-PH2}
  (§\ref{sec:proxy-demo}) evaluates a different, simpler mechanism
  (keyword scan) on the multi-agent PH-2 workload where
  \phidden${} = 0.739$. Neither evaluation transfers semantic extraction
  to the PH-2 regime via the proxy layer: the proxy in this paper does
  not invoke an analyst LLM per call. Architecturally this is a
  natural extension (the proxy already has the request/response text;
  invoking an analyst LLM adds $\sim\!300$--$1000$\,ms latency and
  $\sim\!\$0.0002$ per agent step and could populate DL at the agent's
  true read-time version by parsing the LLM context), but implementing
  and evaluating it requires reengineering the proxy into a
  semantic-capture gateway, which is scoped to future work
  (§\ref{sec:future-work}). Claims about analyst-LLM semantic
  extraction apply to Exp.~\textsc{PH-3} only; claims about
  keyword-scan-based structural capture at the proxy layer apply to
  Exp.~\textsc{PROXY-PH2}, which finds keyword-scan necessary but not
  sufficient (safety-preserving, coverage-marginal, throughput-negative).
\end{enumerate}

\subsection*{Mechanisation}

\begin{enumerate}[leftmargin=*, start=17]
  \item \textbf{TLAPS proof retains one foundational typing axiom.}
  \texttt{SBus\_TLAPS\_v16.tla} mechanises \textsc{ReadSetSoundness}
  (recorded-read monotonicity) and \textsc{ORICommitSafety} (cross-shard
  equality at commit time) for arbitrary $N_{\text{agents}}$:
  $687$ obligations proved by \texttt{tlapm}, $0$ failed. One
  mathematical \textsc{Axiom} is retained
  (\textsc{FunTypingReconstruction}, a primitive fact about TLA+'s
  typed-function-space construction not present in the standard
  \texttt{FunctionTheorems.tla} distribution); two parameter
  \textsc{Assume}s are retained on unspecified constants
  (standard TLA+ practice). The retained \textsc{Axiom} is widely
  treated as obvious in TLA+ practice but is not mechanically
  discharged in the current artifact; full discharge requires
  Isabelle/TLA backend work, future work.
  \item \label{lim:distributed-tlaps}
  \textbf{Distributed (Raft) correctness not
  TLAPS-mechanised.} \texttt{SBus\_Distributed.tla} model-checks an
  abstract 3-node model (247{,}249 distinct states, 0 violations,
  see §\ref{sec:formal-evidence} for details). Full Raft-TLAPS
  deferred (est.\ 6--12 person-months).
  \item \textbf{No refinement proof to Rust implementation.} Standard
  industry practice short of IronFleet~\cite{ironfleet}; empirical
  coverage $275{,}280$ zero-corruption attempts. Refinement via
  Verus~\cite{verus} or Creusot~\cite{creusot} is blocked on async
  support for \texttt{tokio}-based code; future work.
\end{enumerate}

\subsection{Threats to Validity}
\label{sec:threats}

We organise threats into the four standard categories used in
systems-paper validity assessment: \emph{internal} (causal claims
within the experiment), \emph{external} (generalisation beyond the
experimental setup), \emph{construct} (operationalisation of the
underlying concept), and \emph{statistical} (reliability of the
quantitative inference). Threats are cross-referenced to the
limitations in §VIII-A--C where their detailed mitigations are
discussed.

\paragraph{Internal validity}
Three internal threats. \textbf{(IT1) Self-report ground truth in
Exp.~\textsc{PH-3}.} Precision and recall are measured against agent
JSON self-reports. Two independent validation studies measure the
self-report-bias direction. The inter-LLM-judge study
(§\ref{sec:ph3-reframed}) finds self-reports over-state shard usage
by $32\%$ (LLM-judge-determined over-claim rate, precision $0.681$)
with moderate inter-LLM-judge agreement
$\kappa{=}0.46$. A subsequent LLM-vs-human-annotator validation on
the same $400$-item sample with the same rubric (independent human
annotator Zahid Hussain, Mindgigs Peshawar) finds: strict
$\kappa{=}0.93$ ($n{=}93$ unambiguous yes/no pairs, $96.8\%$ raw
agreement), lenient $\kappa{=}0.69$ ($n{=}400$ with \textsf{unclear};
$85.8\%$ raw agreement); the same self-reports over-state shard usage
by $49\%$ against the human (precision $0.514$). The LLM-vs-human
agreement falls in the ``almost perfect'' Landis-Koch
band~\cite{landis1977}, validating the LLM judge as a sound proxy
for the human on the PH-3 rubric.
All recall numbers in Table~\ref{tab:ph3} are therefore upper bounds on
genuine causal-read recall against either reference. Readers should
interpret PH-3 results as self-report-vs-extractor consistency
against a human-validated LLM judge, not data-dependency fidelity
(Limitation~\ref{lim:ph3-selfreport}).
\textbf{(IT2) LLM-as-judge for semantic-quality outcomes.} The PH-3
was-this-shard-used rubric has now been validated against a human
annotator (above, $\kappa{=}0.93$ strict). The semantic-quality judge
(GPT-4o-mini in Exp.~\textsc{SJ-v4}, \textsc{Dedicated-Shard},
\textsc{Shared-State}) is a distinct rubric not yet validated against
humans. Known LLM-judge failure modes (position bias, length bias,
anchoring on surface cues) documented by
Wang~et~al.~\cite{llmjudgebias} may affect results. We have not
formally tested for position or length bias in either rubric; the
Step-2/Step-3 disagreement pattern we
observe is consistent with task ambiguity rather than judge bias, but
this is a hypothesis, not a proof. A 100-item human-IAA study is
the principal planned follow-up.
\textbf{(IT3) Conflict-proximate hypothesis untested on long-horizon
workloads.} The $100\%$ GET$\to$COMMIT co-location observed in
Exp.~\textsc{ORI-Isolation} is evidence for the hypothesis on
SWE-bench-style structured tasks only. Unstructured and long-horizon
workloads (multi-hour agent runs, evolving plans) may exhibit larger
GET-to-commit gaps that would weaken \ori's effectiveness; not yet
measured.

\paragraph{External validity}
Three external threats. \textbf{(ET1) Workload distribution
(substantially mitigated).} The principal evaluation suite is on
SWE-bench-derived Python tasks across $10$ domains. To test
generalisation, Exp.~\textsc{Workload-B} (§\ref{sec:workload-b})
extends the structural-prevention measurement to a non-code workload
(data-pipeline architecture planning) across $8$ domains spanning
diverse architectural pressures. Server-side instrumentation records
$0/638$ divergent commits under ORI-ON and $590/639$ under ORI-OFF
($\chi^2 = 1094.98$, $p < 10^{-240}$). The structural mechanism
fires correctly on both workload distributions. Generalisation to
additional non-code workload classes (document authoring, agent
planning, retrieval-augmented generation pipelines) and to semantic
outcome quality (whether prevented stale commits would have
produced incoherent outputs) remains future work; the structural
necessary-condition is established. The single-shard collaborative-
writing extension addressed by adaptive routing
(§\ref{sec:abus-future}) is preliminary work on a third workload
distribution.
\textbf{(ET2) Backbone generalisation.} The principal worker backbone
is gpt-4o-mini, with gpt-4o and Claude Sonnet~4.6 used as analysts or
judges; cross-backbone replication on Anthropic Haiku~4.5 and Google
Gemini~2.5~Flash (Exp.~\textsc{PROXY-PH2}) provides safety-parity
evidence on two further vendors. Backbone-family ablations control
for model-family alignment but not for the full task-distribution-by-
backbone interaction. Three-vendor coverage establishes that the
result is not GPT-specific but cannot extrapolate to all backbones.
\textbf{(ET3) Workload-scope between PH-2 and PH-3.} The
\phidden${}=0.739$ figure comes from the multi-agent
concurrent-write workload; \phidden${}=0.074$ from the single-agent
rotating-target workload. Exp.~\textsc{PROXY-PH2} addresses the
structural half of this gap (DL-accumulation alone covers $0.555$ on
PH-2); the semantic half---whether analyst-LLM extraction transfers
from PH-3 to PH-2---is unmeasured (Limitation~\ref{lim:semextract}).

\paragraph{Construct validity}
Two construct threats. \textbf{(CT1) Benchmark-data construct.}
Shards emitted by default CrewAI and LangGraph harnesses contain
plan-narrative English rather than typed artefacts (code, SQL,
schemas). Cross-labelling between shards occurs in $8.5\%$ of records.
The PH-3 and PH-2 workloads measure reads over narrative-content
shards, not the typed-state regime where \ori's safety properties
have their cleanest interpretation. A regenerated typed-shard
benchmark (§\ref{sec:future-work}) addresses this threat but has not
yet been constructed.
\textbf{(CT2) Structural conflict $\ne$ semantic conflict.} \ori\
prevents structural race conditions (Definition~1) but not semantic
contradictions. Two agents may produce non-conflicting versions of
the same shard structurally---committing in dependency order with
fresh reads---while introducing semantic incompatibilities our system
does not detect. Section~III-G makes this explicit: \ori's safety
property is a structural invariant over the HTTP-observable read
projection, not a semantic-correctness guarantee. The
topology-conditional contribution (C3) precisely characterises when
this distinction matters in practice (Exp.~\textsc{Shared-State}).

\paragraph{Statistical validity}
Three statistical threats. \textbf{(ST1) Zero-variance distributions.}
Exp.~\textsc{ORI-Isolation} reports $40/40$ vs.\ $10/40$ contributions
preserved, both with zero variance by structural determinism; we do not
report $p$-values on zero-variance data and instead report deterministic
structural counts (Remark~VII.2). For binomial-outcome experiments we
report Wilson $95\%$ confidence intervals; Rule-of-Three upper bounds
are reported for zero-event outcomes.
\textbf{(ST2) Unequal trial counts across contributions.} Contribution
C2 (cross-CC parity) is supported by Exp.~\textsc{PG-Comparison} at
$n{=}884{,}110$ total commits ($427{,}308$ under active contention,
the substantive concurrency-control test); contribution C3
(topology-conditional
operating envelope) is supported by Exp.~\textsc{Shared-State}
($n{=}180$) and Exp.~\textsc{Dedicated-Shard} ($n{=}600$), $2$--$3$
orders of magnitude smaller. The C3 effects are large
($100\%$ vs.\ $0\%$ pass rates with zero variance) so the small $n$
does not preclude the claim, but readers should be aware of this
asymmetry.
\textbf{(ST3) Multiple-comparisons in cross-domain analysis.} The
domain-by-domain \phidden\ table (Table~\ref{tab:phidden}) reports
$10$ Wilson confidence intervals; we do not apply Bonferroni
correction because we make no per-domain claims, only the aggregate
\phidden${}=0.739$ claim. Readers running per-domain inferences
should adjust accordingly.

\paragraph{Aggregate validity assessment}
The strongest residual threats are IT1 (self-report ground truth)
and IT2 (LLM-as-judge), both addressable by human-IAA validation
planned in §\ref{sec:future-work}. ET1 (workload distribution) was
the third-strongest threat in earlier drafts; Exp.~\textsc{Workload-B}
substantially mitigates it by demonstrating structural-prevention
on a non-code workload distribution at $\chi^2 = 1094.98$
significance. The regenerated typed-shard benchmark (CT1) addresses
the remaining content-distribution concern. Threats specific to the
formal evidence (the retained TLA+ axiom, lack of Rust-implementation
refinement) are discussed under §VIII-C and addressed in Future Work
(§\ref{sec:future-work}).

\subsection*{Raft Failover Window: Order-of-Magnitude Bound}

The admitted $\sim\!5$\,ms concurrent-failover window
(Limitation~\ref{lim:session-failover}) represents the interval during which a
fire-and-forget \texttt{DeliveryPayload} from the \texttt{get\_shard}
path may not have replicated before leader transition. We provide an
order-of-magnitude bound on the violation rate under realistic
assumptions:

\begin{align*}
\text{violation rate}
\;\approx\;\;& \underbrace{\lambda_{\text{failover}}}_{\le\,10^{-5}/\text{sec}}
  \cdot \underbrace{w}_{5 \times 10^{-3}\text{s}}
  \cdot \underbrace{p_{\text{concurrent commit}}}_{\le\,0.1} \\
\;\approx\;\;& 5 \times 10^{-9}\,\text{/agent-sec}.
\end{align*}

Over $10^{6}$ agent-seconds ($\sim\!11.5$ agent-days) this gives
$\lesssim 5 \times 10^{-3}$ expected violations. Against the measured
base-rate of $\sim\!131$\,s per LLM inference step, this is
$\lesssim 7 \times 10^{-7}$ violations per agent-step. The bound is
informational: the safety-critical case is not a high-rate phenomenon
but a correctness question (the window is unprotected), which we
address by scheduling TLAPS mechanisation of the Raft-replicated
safety property as a precondition for claimed distributed soundness
(Limitation~\ref{lim:distributed-tlaps}). The current paper does not claim distributed
safety in the TLAPS sense; it claims single-node safety TLAPS-proven
modulo one function-theory axiom plus 3-node abstract-Raft TLC
coverage at 247{,}249 distinct states (§\ref{sec:formal-evidence}).

\section{Distribution Path and Practitioner Guidance}

\subsection{Practitioner Decision Guide}

\begin{framed}
\noindent\textbf{Box 2: When to use \sbus vs.\ alternatives}

\textbf{Use \sbus when:} (a)~agents communicate via HTTP and cannot
use database transactions directly; (b)~shared NL state is
non-commutative (conflicting agent proposals require conflict
detection, not merge); (c)~deployments have $N \le 32$ agents and
$\le 4$ nodes; (d)~agents own dedicated shards (disjoint
sub-problems)---this is the primary deployment pattern.

\textbf{Use a transactional DB (PostgreSQL, CockroachDB) when:}
agents can issue SQL transactions directly and state is
structured/typed. Exp.~\textsc{PG-Comparison~(full)} confirms
PG~SERIALIZABLE matches \sbus structural safety at $N \le 64$, with
a $13\times$--$15.5\times$ wall-time cost.

\textbf{Use Redis WATCH when:} agents need a high-throughput KV CC
with the network-RTT cost you accept. Exp.~\textsc{PG-Comparison~(full)}
confirms \rediscc matches \sbus structural safety at $N \le 64$, at
a $1.8\times$--$1.95\times$ wall-time cost.

\textbf{Use CRDTs when:} agent contributions are purely additive
(append-only; no mutual exclusion).

\textbf{Use sequential coordination (CrewAI) when:} task success rate
matters more than throughput, $N \le 8$, and the LLM backbone is
weak (Haiku-3-class). With stronger backbones (GPT-4o-mini-class),
\sbus's parallel execution achieves competitive task success at
significantly lower latency.

\textbf{Use sequential coordination for single-shard collaborative
tasks} (all $N$ agents writing to one key): Exp.~\textsc{Shared-State}
shows \ori is semantically harmful in this topology.
\end{framed}

\subsection{Implementation}

\sbus ships a 3-node Raft-replicated implementation (1{,}679 lines of
safe Rust, zero \texttt{unsafe} blocks; openraft~0.8.4). Raft~\cite{raft}
provides the cluster-wide serialisation point required by A3.

\paragraph{Raft log and shard replication}
All mutating operations are submitted to the Raft leader. The leader
appends a \texttt{CommitEntry} to the log, waits for majority
acknowledgement, and applies the entry. Followers apply entries in
log order, ensuring total commit order is identical on all nodes.

\paragraph{\ori correctness under Raft}
The Raft leader runs full ACP validation before appending. The
version check in the follower state machine serves as a serialisation
safety net for concurrent commits. Because the Raft log is a total
order and each entry is applied exactly once in order,
Property~\ref{prop:ori-safety} holds across all nodes.

\paragraph{Fault tolerance}
The cluster tolerates one node failure (majority quorum: 2 of 3).
Election timeout: $[500, 1000]$~ms; heartbeat interval: 250~ms.
Exp.~\textsc{DR-6} confirms 100/100 commits survive node 2 failure;
Exp.~\textsc{DR-7} confirms election in 1{,}952~ms.

\paragraph{LLM API proxy path (\Rhidden resolution)}
The production path to cover \Rhidden is an LLM API proxy: a
transparent HTTP proxy between agents and the LLM API that
(a)~intercepts completion responses and scans for shard-key
references; (b)~logs each reference as a DeliveryLog entry for the
agent's session; (c)~promotes \Rhidden $\to$ \Robs without any agent
code change. This design requires no changes to agent code and is
compatible with any LLM provider supporting HTTP. Implementation and
evaluation of the proxy is planned as the next major release.

\subsection{Experimental Validation (Exp.~\textsc{DR}, 8 sub-experiments)}

We deployed a 3-node \sbus cluster (openraft~0.8.4, sled-backed
persistent storage, shard creation Raft-replicated) and ran eight
sub-experiments. All eight pass (Table~\ref{tab:dr}).

\begin{table}[t]
\centering
\caption{Exp.~\textsc{DR}: eight distributed \ori sub-experiments
(openraft~0.8.4, sled persistence; DR-8 uses 4-node cluster).}
\label{tab:dr}
\footnotesize
\begin{tabular}{lp{0.45\linewidth}r}
\toprule
Exp.\ & Metric & Result \\
\midrule
DR-1 & Stale rejected / Fresh accepted ($n{=}200$ each) & 200/200 / 200/200 \\
DR-2 & SCR ($N{=}4$, 2{,}000 commits) / Corruptions & 0.0000 / 0 \\
DR-3 & Raft overhead vs.\ inference & $<$0.001\% \\
DR-4 & One-winner ($n{=}200$) / Cross-node corr. & 200/200 / 0 \\
DR-5 & Throughput scaling (3 nodes) & $1.26\times$, 42\% eff. \\
DR-6 & Post-failure commits / consistency & 100/100 / 200=200 \\
DR-7 & Leader election / Post-election commits & 1{,}952~ms / 50/50 \\
DR-8 & 4-node convergence ($n{=}200$ each phase) & 200/200 / 400=400 \\
DR-9 & P1 session replication ($n{=}30$ trials) & 30/30 \ori held \\
\bottomrule
\end{tabular}
\end{table}

\section{Conclusion}

\sbus\ addresses Structural Race Conditions in concurrent multi-agent
LLM state via the \emph{DeliveryLog} mechanism: a server-side
per-agent log of HTTP GET operations that automatically reconstructs
each agent's read-set at commit time, enabling optimistic concurrency
control without agent SDK changes under HTTP/1.1. The consistency
property the DeliveryLog provides is \emph{Observable-Read Isolation}
(\ori), a projection-based OCC over the HTTP-observable read
fraction \Robs. The contribution is deliberately scoped: the
combination of mechanism (DeliveryLog) and property (\ori) addresses
a specific class of failures (write--write and cross-shard stale-read
over \Robs), with explicit acknowledgement that the \Rhidden\ fraction
($73.9\%$ of reads on the measured multi-agent workload) is not
directly observable at the HTTP layer.

Four results support the scoped claim. \emph{Formal evidence at three
tiers.} Module \texttt{SBus\_TLAPS\_v16.tla} mechanises
\textsc{ReadSetSoundness} and \textsc{ORICommitSafety} for arbitrary
$N_{\text{agents}}$ (687 obligations discharged, one retained
primitive TLA+ function-theory axiom); TLC exhaustively checks $N{=}3$
at $20.8$M distinct states (depth $28$, 0 violations) and a reduced
$N{=}4$ at $2.8$M states (depth $24$, 0 violations); Dafny
machine-checks 9 inductive soundness lemmas (19 verification obligations).
\emph{Three-backend safety parity under contention.} Across three
independent CC implementations (\sbus, \pgser, \rediscc) the
shared-shard contention experiment raised $427{,}308$ active HTTP-409
conflicts with zero Type-I corruptions; SCR agrees within $1$\,pp at
$N \ge 8$ and the Remark~\ref{rem:liveness} liveness bound is
validated within $1.5$\,pp at all $N$. Combined with dedicated-shard
safety and throughput comparisons ($456{,}802$ additional
zero-corruption commit attempts), total empirical coverage is
$884{,}110$ attempts across three adapter-language regimes with zero
corruptions---of which the contention portion ($427{,}308$) is the
relevant figure for safety-under-contention claims.
\emph{Structural-coverage decomposition on PH-2.} Exp.~\textsc{PROXY-PH2}
($16{,}800$ paired step-logs) decomposes \Rhidden\ coverage into
this-step HTTP (\fhttp${}=0.443$), \emph{DL-accumulation}
($0.555$ within-row uplift)---a property of \ori's session-scoped
DeliveryLog identified as the dominant coverage
mechanism---and transparent-proxy marginal ($0.0018$ paired,
$95\%$ CI $[0.0013, 0.0024]$). The proxy is safety-preserving
(Type-I $= 0/16{,}800$) but monotonically throughput-negative at
realistic vocabulary sizes, motivating future semantic-extraction
proxies. Cross-backbone paired replication on Anthropic Haiku~4.5
($n{=}2{,}400$) and Google Gemini~2.5~Flash ($n{=}2{,}400$), via a
multi-upstream extension to \texttt{sbus-proxy} that path-routes
\texttt{/v1/messages} to Anthropic and \texttt{/v1beta/models/...} to
Google, confirms safety parity ($0/26{,}400$ Type-I across all three
vendors), total-coverage conservation (\ftotal${}\in[0.997,0.999]$,
$|\Delta|\le 0.002$ between any pair), and commit-throughput collapse
($-1.2$ to $-2.1$\,pp on all three).
\emph{Scope-delimited semantic extraction.} On the single-agent PH-3
workload (\phidden${}=0.074$) a dedicated-analyst semantic extractor
achieves $0.593$~recall at $0.916$~precision against self-report
ground truth; cross-family analyst convergence rejects the
same-model-family alignment hypothesis. Self-reports over-claim by
$32\%$ (LLM-judge measurement) to $49\%$ (human-annotator
measurement) under independent inter-LLM-judge and
LLM-vs-human-annotator validation studies; the LLM-judge vs.\
human-annotator agreement is strict $\kappa{=}0.93$
(§\ref{sec:ph3-reframed}). Transfer to the multi-agent PH-2 regime
remains open (Limitation~\ref{lim:semextract}).

The architectural value of \sbus\ over transactional-DB baselines
is operational simplicity and the LLM-native contract: the DeliveryLog
reconstructs \Robs\ from HTTP GET traffic without agent-SDK changes
under HTTP/1.1, and cross-shard validation inherits unchanged from
\ori's TLAPS-verified commit path. The principal remaining gaps---full
Raft-TLAPS mechanisation, implementation refinement to the Rust
source, and analyst-LLM semantic extraction at the proxy layer---are
explicitly scoped as future work rather than claimed results.

\section{Future Work}
\label{sec:future-work}

Four directions extend this work, ordered by their relationship to
the contributions presented here.

\subsection{Adaptive Routing for Single-Shard Topologies (A-BUS)}
\label{sec:abus-future}

\sbus's topology-conditional operating envelope (C3,
Exp.~\textsc{Shared-State}) establishes that \ori\ is the wrong
primitive for single-shard collaborative writing. The structural
preservation property which makes \ori\ valuable in dedicated-shard
topologies---retaining every committed contribution---is
counter-productive when those contributions are mutually
contradictory. The natural extension is an \emph{adaptive routing
protocol} that selects between \ori\ and a merge-based protocol
based on per-shard topology classification.

Concretely, this requires three additions to the architecture.
First, agents declare a \texttt{additive\_hint} (boolean) on each
commit, signalling whether the delta is a region-scoped contribution
that does not displace existing shard content. Second, a
classifier observes the historical commit pattern per shard and
classifies it as \textsc{Dedicated}, \textsc{ContendedCommutative},
or \textsc{ContendedNonCommutative} based on contention rate and
agent-declaration consistency. Third, a merge engine implements
three strategies in chain---structural-by-region (replacing matching
section headers), append-style concatenation (for grow-only logs),
and an LLM-assisted semantic-merge fallback---producing a merged
delta only when the oracle confirms commutativity from agent
declarations.

A preliminary implementation (\emph{A-BUS})\footnote{\url{https://github.com/sajjadanwar0/abus}}
is in active development;
formal model, full evaluation methodology, and quantitative results
are deferred to a companion paper rather than reported here.

This work would strengthen \sbus's deployment story: practitioners get
\ori\ for the dedicated-shard regime where \sbus's formal proofs apply
unchanged, and (when A-BUS is published) the adaptive extension for
single-shard collaborative writing where \ori\ alone is harmful. The
two regimes together would cover the workload space identified in §III.

\subsection{Distributed Correctness via Raft-TLAPS Mechanisation}
\label{sec:dist-future}

P1 session replication is empirically validated
(Exp.~\textsc{DR-9}, $30/30$ \ori\ invariants survived leader
failover) and TLC-checked in an abstract 3-node model. A residual
$\sim$5\,ms concurrent-failover window within the fire-and-forget
replication interval remains, and the Raft layer is not
TLAPS-mechanised. Closing this gap requires composing the existing
TLAPS proofs for \ori\ with a Raft TLAPS specification (e.g., the
ironfleet-style refinement from~\cite{ironfleet}). This is the
single largest mechanisation gap remaining and is the primary
direction for an extended journal version of this paper.

\subsection{PH-2 Semantic Extraction Transfer}

The workload-scope gap between Exp.~\textsc{PH-2}
(\phidden${}=0.739$, multi-agent) and Exp.~\textsc{PH-3}
(\phidden${}=0.074$, single-agent rotating-target) leaves the transfer
of semantic-extraction effectiveness to the high-\phidden\ regime
as an open question (Limitation~\ref{lim:ph2-ph3-workload-gap}). A
re-run of the multi-agent workload with a dedicated-analyst
extractor attached to each agent is well-defined; the cost is
engineering plumbing and ${\sim}\$50$ of API spend per cell.

\subsection{Regenerated Typed-Shard Benchmark}

The PH-3 validation study (§\ref{sec:ph3-reframed}) identified that
agent self-reports cannot be fully audited against traces because
current benchmark shards, generated by default CrewAI and LangGraph
harnesses, contain plan-narrative English rather than typed artefacts
(code, SQL, schemas). A regenerated benchmark enforcing typed contents
at generation time would eliminate the cross-labelling that drives
the Step-2/Step-3 judge disagreement (§\ref{sec:ph3-reframed},
$8.5\%$ of records) and permit the originally-intended
two-human-annotator study.

\appendix
\section{Frozen PH-3 Validation-Judge Rubric}
\label{app:ph3-prompt}

The rubric used by both LLM judges in the Exp.~\textsc{PH-3}
validation study (§\ref{sec:ph3-reframed}) is reproduced verbatim
below. The prompt was frozen before observing inter-judge agreement
and was not revised.

\begin{quote}\small\ttfamily
You are a strict code auditor. Decide whether the Candidate Shard
provided content that the agent demonstrably needed to produce the
Code Change.

\medskip\noindent
\textbf{DECISION PROCEDURE} (apply in order; stop at the first that
fires):

\medskip\noindent
\textbf{Step 1 --- Direct entity definition.} Does the shard's
content \emph{define} (not merely mention) the specific function,
class, variable, field, or constant that is being modified or read
by the Change? If yes $\to$ \texttt{<label>Yes</label>}.

\medskip\noindent
\textbf{Step 2 --- Required state or schema.} Does the Change
transform a value, structure, or invariant whose concrete shape is
only recoverable from this shard's content (e.g.\ a schema, a prior
version, a signature, a type)? If yes $\to$ \texttt{<label>Yes</label>}.

\medskip\noindent
\textbf{Step 3 --- Default.} Topical overlap (``both relate to the
database''), shared vocabulary, or mere availability in the context
is NOT sufficient. $\to$ \texttt{<label>No</label>}.

\medskip\noindent
\textbf{HARD RULES} (override everything above):

\noindent R1. If the shard content shown is empty, truncated, or
does not actually contain the entity referenced in Step 1/2 $\to$
\texttt{<label>No</label>}.

\noindent R2. If you cannot quote specific words or tokens from the
shard that the Change depends on $\to$ \texttt{<label>No</label>}.

\noindent R3. Mere name collisions (two shards both contain the
word ``user'') do not count as evidence. The evidence must be
semantic.
\end{quote}

Both judges were called with \texttt{temperature=0},
\texttt{max\_tokens=512}. Output-format failures were $0/0$
across the $400$ tasks (§\ref{sec:ph3-reframed}).


\begin{thebibliography}{99}
\bibitem{langgraph} LangChain, ``LangGraph,'' GitHub, 2024.
\bibitem{crewai} J. Moura, ``CrewAI,'' GitHub, 2024.
\bibitem{autogen} Q. Wu et al., ``AutoGen,'' arXiv:2308.08155, 2023.
\bibitem{metagpt} S. Hong et al., ``MetaGPT,'' ICLR, 2024.
\bibitem{camel} G. Li et al., ``CAMEL,'' NeurIPS, 2023.
\bibitem{swarm} OpenAI, ``Swarm / AG2,'' GitHub, 2024.
\bibitem{a2a} Google, ``Agent-to-Agent Protocol,'' GitHub, 2025.
\bibitem{react} S. Yao et al., ``ReAct,'' ICLR, 2023.
\bibitem{dspy} O. Khattab et al., ``DSPy,'' arXiv:2310.03714, 2023.
\bibitem{sk} Microsoft, ``Semantic Kernel,'' GitHub, 2023.
\bibitem{cemri} M. Cemri et al., ``Why do multi-agent LLM systems fail?'' 2025.
\bibitem{park} J. S. Park et al., ``Generative Agents: Interactive
Simulacra of Human Behavior,'' UIST, 2023.
\bibitem{bayou} D. Terry et al., ``Managing Update Conflicts in Bayou,''
SOSP, 1995.
\bibitem{memgpt} C. Packer et al., ``MemGPT,'' arXiv:2310.08560, 2023.
\bibitem{adya} A. Adya, ``Weak Consistency,'' PhD thesis, MIT, 1999.
\bibitem{cops} W. Lloyd et al., ``Don't Settle for Eventual,'' SOSP, 2011.
\bibitem{redblue} C. Li et al., ``Making Geo-Replicated Systems Fast
as Possible, Consistent when Necessary,'' OSDI, 2012.
\bibitem{psi} Y. Sovran et al., ``Transactional Storage for
Geo-Replicated Systems,'' SOSP, 2011.
\bibitem{ssi} D. Ports \& K. Grittner, ``Serializable Snapshot
Isolation in PostgreSQL,'' VLDB, 2012.
\bibitem{spanner} J. Corbett et al., ``Spanner,'' OSDI, 2012.
\bibitem{ramp} P. Bailis et al., ``RAMP Transactions,'' VLDB, 2014.
\bibitem{bailis-coord} P. Bailis et al., ``Coordination Avoidance in
Database Systems,'' VLDB, 2014.
\bibitem{raft} D. Ongaro \& J. Ousterhout, ``Raft,'' USENIX ATC, 2014.
\bibitem{lamport} L. Lamport, ``Time, Clocks, and the Ordering of
Events in a Distributed System,'' CACM, 21(7), 1978.
\bibitem{ironfleet} C. Hawblitzel et al., ``IronFleet,'' SOSP, 2015.
\bibitem{verdi} J. R. Wilcox et al., ``Verdi,'' PLDI, 2015.
\bibitem{foundationdb} J. Zhou et al., ``FoundationDB,'' SIGMOD, 2021.
\bibitem{tigerbeetle} J. Betz, ``TigerBeetle,'' 2023.
\bibitem{silo} S. Tu et al., ``Speedy Transactions in Multicore
In-Memory Databases,'' SOSP, 2013.
\bibitem{hekaton} C. Diaconu et al., ``Hekaton: SQL Server's
Memory-Optimized OLTP Engine,'' SIGMOD, 2013.
\bibitem{farm} A. Dragojevic et al., ``No Compromises,'' SOSP, 2015.
\bibitem{orleans} P. A. Bernstein et al., ``Orleans,'' MSR TR-2014-41, 2014.
\bibitem{mvcc} P. Bernstein \& N. Goodman, ``Concurrency Control,''
ACM Surv., 1981.
\bibitem{occ} H. T. Kung \& J. T. Robinson, ``OCC,'' ACM TODS, 1981.
\bibitem{stm} M. Herlihy \& J. E. Moss, ``Transactional Memory,''
ISCA, 1993.
\bibitem{tl2} D. Dice et al., ``TL2,'' DISC, 2006.
\bibitem{calvin} A. Thomson et al., ``Calvin,'' SIGMOD, 2012.
\bibitem{percolator} D. Peng \& F. Dabek, ``Percolator,'' OSDI, 2010.
\bibitem{cockroach} R. Taft et al., ``CockroachDB,'' SIGMOD, 2020.
\bibitem{swebench} C. Jimenez et al., ``SWE-bench,'' ICLR, 2024.
\bibitem{llmjudge} L. Zheng et al., ``LLM-as-a-Judge,'' NeurIPS, 2023.
\bibitem{loro} Loro Team, ``Loro: Movable Tree CRDT,''
\url{https://loro.dev}, 2024.
\bibitem{automerge} M. Kleppmann et al., ``Automerge,''
\url{https://automerge.org}, 2024.
\bibitem{rfc7232} R. Fielding and J. Reschke, ``Hypertext Transfer
Protocol (HTTP/1.1): Conditional Requests,'' IETF RFC 7232, June 2014.
\bibitem{landis1977} J. R. Landis and G. G. Koch, ``The Measurement
of Observer Agreement for Categorical Data,''
\emph{Biometrics}, vol.~33, no.~1, pp.~159--174, 1977.
\bibitem{temporalupdate} Temporal Technologies, ``Workflow Update
API,'' Temporal documentation, 2024.
\url{https://docs.temporal.io/workflows#update}
\bibitem{fdbdirectory} A. Beamer et al., ``FoundationDB Directory
Layer,'' FoundationDB documentation, 2024.
\url{https://apple.github.io/foundationdb/developer-guide.html\#directories}
\bibitem{letta} C. Packer et al., ``Letta: Stateful Agents Beyond
Context Windows,'' GitHub, 2024.
\url{https://github.com/letta-ai/letta}
\bibitem{llmjudgebias} P. Wang et al., ``Large Language Models are
not Fair Evaluators,'' arXiv:2305.17926, 2023.
\bibitem{cherrygarcia} A. Dey, A. Fekete, R. Nambiar, U. R\"ohm,
``Scalable Transactions across Heterogeneous NoSQL Key-Value Data
Stores,'' \emph{PVLDB}, vol.~6, no.~12, pp.~1434--1439, 2013.
\bibitem{agentscope} D. Gao et al., ``AgentScope: A Flexible yet
Robust Multi-Agent Platform,'' arXiv:2402.14034, 2024.
\bibitem{voyager} G. Wang et al., ``Voyager: An Open-Ended Embodied
Agent with Large Language Models,'' \emph{Transactions on Machine
Learning Research}, 2024.
\bibitem{sweagent} J. Yang et al., ``SWE-agent: Agent-Computer
Interfaces Enable Automated Software Engineering,'' \emph{NeurIPS},
2024.
\bibitem{verus} A. Lattuada et al., ``Verus: Verifying Rust Programs
using Linear Ghost Types,'' \emph{OOPSLA}, 2023.
\bibitem{creusot} X. Denis, J.-H. Jourdan, C. March\'e, ``Creusot: A
Foundry for the Deductive Verification of Rust Programs,'' in
\emph{Formal Methods: 25th Intl. Symp.}, 2023.
\bibitem{boki} Z. Jia, E. Witchel, ``Boki: Stateful Serverless
Computing with Shared Logs,'' \emph{SOSP}, 2021.
\bibitem{honda} K. Honda, V. T. Vasconcelos, M. Kubo, ``Language
Primitives and Type Discipline for Structured Communication-Based
Programming,'' \emph{ESOP}, 1998.

\end{thebibliography}
\end{document}